\documentclass{article}
\usepackage{proceed2e}
\usepackage{amsthm}
\usepackage{times}

\usepackage{epsfig}
\usepackage{amsmath}

\usepackage{makeidx}         
\usepackage{graphicx}        
\usepackage{amssymb}                             
\usepackage{multicol}        

\newtheorem{thm}{Theorem}

{\itshape}{\rmfamily}
\newtheorem{assum}{Assumption}
\usepackage{url}

\usepackage{amsmath}
\usepackage{amsfonts}
\newcommand{\e}[1]{E\left[ #1 \right]}

\usepackage{algorithm,algorithmic}

\usepackage{times}

\usepackage[toc,page]{appendix}

\title{Stochastic Discriminative EM}

\author{
\begin{tabular}{c@{\hspace{10mm}}c}
\multicolumn{2}{c}{Andr\'es R. Masegosa$^{\dagger\ddagger}$} \\\\
\textnormal{$^\dagger$ Dept. of Computer and Information Science} &  \textnormal{$^\ddagger$ Dept. of Computer Science and A. I.} \\
\textnormal{Norwegian University of Science and Technology} & \textnormal{University of Granada} \\
\textnormal{Trondheim, Norway }& \textnormal{Granada, Spain}\\
\end{tabular}
}

\begin{document}
\maketitle

\begin{abstract}
Stochastic discriminative EM (sdEM) is an online-EM-type algorithm for discriminative training of probabilistic generative models belonging to the exponential family. In this work, we introduce and justify this algorithm as a stochastic \emph{natural gradient} descent method, i.e. a method which accounts for the information geometry in the parameter space of the statistical model. We show how this learning algorithm can be used to train probabilistic generative models by minimizing different discriminative loss functions, such as the negative conditional log-likelihood and the Hinge loss. The resulting models trained by sdEM are always generative (i.e. they define a joint probability distribution) and, in consequence, allows to deal with missing data and latent variables in a principled way either when being learned or when making predictions. The performance of this method is illustrated by several text classification problems for which a multinomial naive Bayes and a latent Dirichlet allocation based classifier are learned using different discriminative loss functions.
\end{abstract}

\section{INTRODUCTION}

Online learning methods based on stochastic approximation theory \cite{kushner1997stochastic}
have been a promising research direction to tackle the
learning problems of the so-called Big Data era \cite{ahn2012bayesian,bottou2010large,bottou2007tradeoffs}. Stochastic
gradient descent (SGD) is probably the best known example of this kind of
techniques, used to solve a wide range of learning
problems \cite{bottou1998online}. This algorithm and
other versions \cite{ruppert1988efficient} are usually employed to train discriminative models such as
logistic regression or SVM \cite{bottou2010large}.

There also are some successful examples of the use of SGD for discriminative training of probabilistic generative models, as is the case of deep belief networks \cite{hinton2006fast}. However, this learning algorithm cannot be used directly for the discriminative training of general generative
models. One of the main reasons is that statistical estimation or risk minimization problems of generative models involve the solution of an optimization problem with a large number of \emph{normalization constraints} \cite{PernkopfWT12}, i.e. those which guarantee that the optimized parameter set defines a valid probabilistic model. Although successful solutions to this problem have been proposed  \cite{greiner2005structural,lacoste2008disclda,PernkopfWT12,zhu2009medlda}, they are based on ad-hoc methods which cannot be easily extended to other statistical models, and hardly scale to large data sets. 


Stochastic approximation theory \cite{kushner1997stochastic} has also	
been used for 	maximum likelihood estimation (MLE) of probabilistic generative models
with latent variables, as is the case of the online
EM algorithm \cite{cappe2009line,sato2000convergence}. This method provides
efficient MLE estimation for a broad class of statistical models (i.e. exponential family models)
by sequentially updating the so-called \emph{expectation parameters}.  The advantage of this approach is
that the resulting iterative optimization algorithm is fairly simple and amenable, as it does not
involve any normalization constraints.  

In this paper we show that the derivation of Sato's online EM \cite{sato2000convergence}
can be extended for the discriminative learning of generative models
by introducing a novel interpretation of this algorithm as a natural
gradient algorithm \cite{amari1998natural}.
The resulting algorithm, called stochastic discriminative EM
(sdEM), is an online-EM-type algorithm that can train generative
probabilistic models belonging to the exponential family
using a wide range of discriminative loss functions, such as the
negative conditional log-likelihood or the Hinge loss. In opposite to other discriminative learning approaches \cite{PernkopfWT12}, models trained by sdEM can deal with missing data and latent variables in a principled way either when being learned or when making predictions, because at any moment they always define a joint probability distribution. sdEM could be used for learning using large scale data sets due to its stochastic approximation nature and, as we will show, because it allows to compute the \emph{natural gradient} of the loss function with no extra cost  \cite{amari1998natural}. Moreover, if allowed by the generative model and the discriminative loss function, the presented algorithm could potentially be used interchangeably for classification or regression or any other prediction task. But in this initial work, sdEM is only experimentally evaluated in classification problems.

The rest of this paper is organized as follows. Section 2 provides the preliminaries for the description 
of the sdEM algorithm, which is detailed in Section 3. A brief experimental evaluation is given in Section 4, while
Section 5 contains the main conclusions of this work.

\section{PRELIMINARIES}
\label{Section:Preliminaries}

\subsection{MODEL AND ASSUMPTIONS}
\label{Section:Models}


We consider generative statistical models for prediction tasks,
where $Y$ denotes the random variable (or the vector-value random
variable) to be predicted, $X$ denotes the predictive variables,
and $y^\star$ denotes a prediction, which is made according to $y^\star = arg\max_y p(y,x|\theta)$.
\begin{assum}\label{Assumption:NEF}
The generative data model belongs to the exponential family with a natural (or canonical) parametrization
\begin{equation*}
p(y,x|\theta) \propto exp(\langle s(y,x), \theta\rangle -
A_l(\theta) )
\end{equation*}
\noindent where $\theta$ is the so-called natural parameter which
belongs to the so-called \emph{natural parameter space}
$\Theta\in\Re^K$, $s(y,x)$ is the
vector of sufficient statistics belonging to a convex set
$\mathcal{S} \subseteq\Re^K$, $\langle\cdot,\cdot\rangle$
denotes the dot product and $A_l$ is the log partition function.
\end{assum}
\begin{assum}\label{Assumption:ConjugatePrior}
We are given a conjugate prior distribution $p(\theta|\alpha)$ of
the generative data model
\begin{equation*}
p(\theta|\alpha) \propto exp(\langle s(\theta), \alpha \rangle -
A_g(\alpha))
\end{equation*}
\noindent where the sufficient statistics are
$s(\theta)=(\theta,-A_l(\theta))$ and the hyperparameter $\alpha$
has two components $(\bar{\alpha}, \nu)$. $\nu$ is a positive
scalar and $\bar{\alpha}$ is a vector also belonging to
$\mathcal{S}$ \cite{bernardo2009bayesian}.

\end{assum}

\subsection{DUAL PARAMETERIZATION AND ASSUMPTIONS}
\label{Section:DualParameters} 
The so-called \emph{expectation
parameter}
$\mu\in\mathcal{S}$ can also be used to parameterize probability
distributions of the exponential family. It is a dual set
of the model parameter $\theta$ \cite{amari1985differential}. This
\emph{expectation parameter} $\mu$ is defined as the expected
vector of sufficient statistics with respect to $\theta$:
\begin{equation}\label{Equation:1}
\begin{array}{lll}
\mu & \triangleq & \e{s(y,x)|\theta}  = \int s(y,x)p(y,x|\theta)
dydx\\
& = &\partial A_l(\theta)/\partial \theta
\end{array}
\end{equation}
The transformation between $\theta$ and $\mu$ is one-to-one: $\mu$
is a dual set of the model parameter $\theta$ \cite{amari1985differential}.
Therefore, Equation (\ref{Equation:1}) can be inverted as: $\theta
= \theta(\mu)$.
That is to say, for each $\theta\in \Theta$ we always have an
associated $\mu\in\mathcal{S}$ and both parameterize the same
probability distribution.

For obtaining the \emph{natural parameter} $\theta$ associated to
an \emph{expectation parameter} $\mu$, we need to make use of the
negative of the entropy,
\begin{equation}\label{Equation:2}
\begin{array}{lll}
H(\mu) & \triangleq & \int p(y,x|\theta(\mu)) \ln p(y,x|\theta(\mu)) dydx\\
&=& \sup_{\theta\in\Theta} \langle \mu, \theta\rangle -
A_l(\theta)
\end{array}
\end{equation}
Using the above function, the natural parameter $\theta$ can be
explicitly expressed as
\begin{equation}\label{Equation:3} \theta =
\theta(\mu) = \partial H(\mu)/\partial \mu
\end{equation}
Equations (\ref{Equation:1}), (\ref{Equation:2}),
(\ref{Equation:3}) define the Legendre-Fenchel transform.

Another key requirement of our approach is that it should be possible to 
compute the transformation from $\mu$ to $\theta$ in closed form:
\begin{assum}\label{Assumption:ML}
The transformation from the expectation parameter $\mu$ to the
natural parameter $\theta$, which can be expressed as
\begin{eqnarray}
\theta(\mu) = arg\max_{\theta\in\Theta} \langle \mu, \theta\rangle
-A_l(\theta)
\end{eqnarray}
is available in closed form.
\end{assum}
The above equation is also known as the \emph{maximum likelihood
function}, because $\theta(\frac{1}{n}\sum_{i=1}^n s(y_i,x_i))$
gives the maximum likelihood estimation $\theta^\star$ for a data
set with $n$ observations $\{(y_1,x_1),\ldots,(y_n,x_n)\}$.

For later convenience, we show the following relations between the
Fisher Information matrices $I(\theta)$ and $I(\mu)$ for the
probability distributions $p(y,x|\theta)$ and
$p(y,x|\theta(\mu))$, respectively \cite{nielsen2009statistical}:
\begin{eqnarray}\label{Equation:ExponentialIdentities}
I(\theta) & = &\frac{\partial^2 A_l(\theta)}{\partial
\theta\partial\theta} =
\frac{\partial \mu}{\partial \theta}  = I(\mu)^{-1}\\
I(\mu) &= &\frac{\partial^2 H(\mu)}{\partial\mu\partial\mu}  =  \frac{\partial \theta}{\partial \mu} = I(\theta)^{-1}
\end{eqnarray}

\subsection{THE NATURAL GRADIENT}
\label{Section:NaturalGradient}
Let $\mathcal{W}=\{w\in \Re^K\}$
be a parameter space on which the function $L(w)$ is defined. When
$\mathcal{W}$ is a Euclidean space with an orthonormal coordinate
system, the negative gradient points in the direction of steepest
descent. That is, the negative gradient $-\partial L(w)/\partial
w$ points in the same direction as the solution to:
\begin{equation}\label{Equation:EuclideanGradient}
arg\min_{dw} L(w +
dw)\hspace{2mm}subject\hspace{1mm}to\hspace{2mm}||dw||^2 =
\epsilon^2
\end{equation}
\noindent for sufficiently small $\epsilon$, where $||dw||^2$
is the squared length of a small increment vector $dw$ connecting
$w$ and $w+dw$.
This justifies the use of the classical gradient descent method
for finding the minimum of $L(w)$ by taking steps (of size $\rho$) in
the direction of the negative gradient:
\begin{equation}\label{Equation:GradientDescent}
w_{t+1} = w_{t} - \rho \frac{\partial L(w_{t})}{\partial w}
\end{equation}
However, when $\mathcal{W}$ is a Riemannian space \cite{amari2007methods}, there are no
orthonormal linear coordinates, and the squared length of vector
$dw$ is defined by the following equation,
\begin{equation}\label{Equation:RiemannianDistance}
||dw||^2 =  \sum_{ij} g_{ij}(w)dw_idw_j
\end{equation}
\noindent where the $K\times K$ matrix $G=(g_{ij})$ is called the
Riemannian metric tensor, and it generally depends on $w$. $G$
reduces to the identity matrix in the case of the Euclidean space \cite{amari2007methods}.

In a Riemannian space, the steepest descent direction is not
anymore the traditional gradient. That is, $-\partial
L(w)/\partial w$ is not the solution of Equation
(\ref{Equation:EuclideanGradient}) when the squared length of the
distance of $dw$ is defined by Equation
(\ref{Equation:RiemannianDistance}). Amari \cite{amari1998natural} shows that this
solution can be computed by pre-multiplying the traditional
gradient by the inverse of the Riemannian metric $G^{-1}$,
\begin{thm}\label{Theorem:Amari} The steepest descent direction
or the natural gradient of $L(w)$ in a Riemannian space is given
by
\begin{equation}
-\frac{\tilde{\partial} L(w)}{\tilde{\partial}w} =
-G^{-1}(w)\frac{\partial L(w)}{\partial w}
\end{equation}
\end{thm}
\noindent where $\tilde{\partial} L(w)/\tilde{\partial}w$ denotes
the \emph{natural gradient}.

As argued in \cite{amari1998natural}, in statistical estimation problems we
should used gradient descent methods which account for the natural
gradient of the parameter space, as the parameter space of a statistical model (belonging to the
exponential family or not) is a Riemannian space with the Fisher
information matrix of the statistical model $I(w)$ as the tensor
metric \cite{amari1985differential}, and this is the only invariant metric that must be given to the
statistical model \cite{amari1985differential}.

\subsection{SATO'S ONLINE EM ALGORITHM}\label{Section:OnlineEM}

Sato's online EM algorithm \cite{sato2000convergence}
is used for maximum likelihood estimation of
missing data-type statistical models. The model defines a
probability distribution over two random or vector-valued
variables $X$ and $Z$, and is assumed to belong to the exponential family:
\begin{equation*}
p(z,x|\theta) \propto exp(\langle s(z,x), \theta\rangle -
A_l(\theta) )
\end{equation*}
\noindent where $(z,x)$ denotes a so-called \emph{complete data}
event. The key aspect is that we can only observe $x$, since
$z$ is an unobservable event. In consequence, the loss function
$\ell(x,\theta)$\footnote{We derive this algorithm in terms of
minimization of a loss function to highlight its connection with
sdEM.} is defined by marginalizing $z$: $\ell(x,\theta) =
-\ln \int p(z,x) dz$.

The online setting assumes the observation of a non-finite data sequence
 $\{(x_t)\}_{t\geq 0}$ independently drawn
according to the unknown data distribution $\pi$. The objective
function that EM seeks to minimize is given by the following
expectation: $L(\theta) = \e{\ell(x,\theta)|\pi}$.

Sato \cite{sato2000convergence} derived the stochastic updating equation of online EM by relying on the
free energy formulation, or lower bound maximization, of the EM algorithm \cite{neal1998view} and on a discounting
averaging method. Using our own notation, this updating equation is
expressed as follows,
\begin{align}\label{Equation:SatoIteration}
\mu_{t+1} &=& (1-\rho_t)\mu_t + \rho_t
E_z[s(z,x_t|\theta(\mu_t)]\nonumber\\ &=& \mu_t + \rho_t
\left(E_z[s(z,x_t|\theta(\mu_t)] - \mu_t\right)\nonumber\\ &=&
\mu_t + \rho_t \frac{\partial
\ell(x_t,\theta(\mu_t))}{\partial\theta}
\end{align}
\noindent where $E_z[s(z,x_t|\theta(\mu_t)]$ denotes the expected
sufficient statistics, $E_z[s(z,x_t|\theta(\mu_t)]=\int
s(z,x_t)p(z|x_t,\theta(\mu_t))dz$.

He proved the convergence of the above iteration method by
casting it as a second order stochastic gradient descent using
the following equality,
\begin{equation}\label{Equation:Sato}
\frac{\partial \ell(x,\theta)}{\partial \theta} = \frac{\partial
\mu}{\partial \theta} \frac{\partial
\ell(x,\theta(\mu))}{\partial \mu}\\ = I(\mu)^{-1}\frac{\partial
\ell(x,\theta(\mu))}{\partial \mu}
\end{equation}
This equality is obtained by firstly applying the chain rule, followed 
by the equality shown in Equation
(\ref{Equation:ExponentialIdentities}). It shows that online EM is
equivalent to a stochastic gradient descent with $I(\mu_t)^{-1}$ as
coefficient matrices \cite{bottou1998online}.

Sato noted that that the third term of the equality in Equation
(\ref{Equation:Sato}) resembles a natural gradient (see Theorem \ref{Theorem:Amari}), but
he did not explore the connection. But the key insights of the
above derivation, which were not noted by Sato, is
that Equation (\ref{Equation:Sato}) is also valid for other loss
functions different from the marginal log-likelihood; and
that the convergence of Equation (\ref{Equation:SatoIteration})
does not depend on the formulation of the EM as a ``lower bound
maximization'' method \cite{neal1998view}.

\section{STOCHASTIC DISCRIMINATIVE EM}\label{Section:sdEM}

\subsection{THE sdEM ALGORITHM}\label{Section:sdEMAlgorithm}

We consider the following supervised learning setup. Let us assume that we
are given a data set $D$ with $n$ observations
$\{(y_1,x_1),\ldots,(y_n,x_n)\}$. We are also given a
\emph{discriminative loss function}\footnote{The loss function is
assumed to satisfy the mild conditions given in \cite{bottou1998online}. E.g., it can
be a non-smooth function, such as the Hinge Loss.}
$\ell(y_i,x_i,\theta)$. For example, it could be the negative
conditional log-likelihood (NCLL) $\ell(y_i,x_i,\theta)=-\ln
p(y_i,x_i|\theta) + \ln \int p(y,x_i|\theta)dy = -\ln
p(y_i|x_i,\theta)$. Our learning problem consists in minimizing 
the following objective function:
\begin{eqnarray}
L(\theta) & = & \sum_{i=1}^n \ell(y_i,x_i,\theta) - \ln
p(\theta|\alpha)\nonumber
\\
& = & \e{\ell(y,x,\theta)|\pi} - \frac{1}{n}\ln
p(\theta|\alpha)\label{Equation:sdEMPiorLoss}
\end{eqnarray}
\noindent where $\pi$ is now the empirical distribution of $D$
and $\e{\ell(y,x,\theta)|\pi}$ the empirical risk. Although the
above loss function is not standard in the machine learning
literature, we note that when $\ell$ is the negative
log-likelihood (NLL), we get the classic \emph{maximum a posterior
estimation}. This objective function can be seen as an extension
of this framework.

sdEM is presented as a generalization of Sato's online EM
algorithm for finding the minimum of an objective function in the
form of Equation (\ref{Equation:sdEMPiorLoss}) (i.e. the solution
to our learning problem). The stochastic updating equation of sdEM can be expressed as follows,
\begin{equation}\label{Equation:sdEMSecondOrder}
\mu_{t+1} = \mu_{t} - \rho_t I(\mu_{t})^{-1}\frac{\partial
\bar{\ell}(y_t,x_t,\theta(\mu_{t}))}{\partial \mu}
\end{equation}
\noindent where  $(y_t,x_t)$ denotes the $t$-th sample, randomly generated from
$\pi$, and the function $\bar{\ell}$ has the following
expression: $\bar{\ell}(y_t,x_t,\theta(\mu_{t})) =
\ell((y_t,x_t,\theta(\mu_{t})) + 1/n\ln p(\theta(\mu_{t}))$. We note that this loss function
satisfies the following equality, which is the base for a stochastic approximation
method \cite{kushner1997stochastic}, $\e{\bar{\ell}(y_t,x_t,\theta(\mu))|\pi}=L(\theta(\mu))$.

Similarly to Amari's natural gradient algorithm \cite{amari1998natural}, the main problem of
sdEM formulated as in Equation (\ref{Equation:sdEMSecondOrder}) is
the computation of the inverse of the Fisher information matrix at
each step, which becomes even prohibitive for large models. The following result shows that this can be
circumvented when we deal with distributions of the exponential family:

\begin{thm}\label{Theorem:Sato}
In the exponential family, the natural gradient of a loss
function with respect to the expectation parameters equals the
gradient of the loss function with respect to the natural
parameters,
\begin{equation*}
I(\mu)^{-1}\frac{\partial \bar{\ell}(y,x,\theta(\mu))}{\partial
\mu} = \frac{\partial \bar{\ell}(y,x,\theta)}{\partial \theta}
\end{equation*}
\end{thm}
\begin{proof}[Sketch of the proof]
We firstly need to prove that $I(\mu)$ is a valid Riemannian
tensor metric and, hence, the expectation parameter space has a
Riemanian structure defined by the metric $I(\mu)$ and the
definition of the natural gradient makes sense. This can be proved
by the invariant property of the Fisher information metric to
one-to-one reparameterizations or, equivalently, transformations in
the system of coordinates \cite{amari1985differential,amari2007methods}. $I(\mu)$ is a
Riemannian metric because it is the Fisher information matrix of
the reparameterized model $p(y,x|\theta(\mu))$, and the
reparameterization is one-to-one, as commented in Section
\ref{Section:DualParameters}.

The equality stated in the theorem follows directly from
Sato's derivation of the online EM algorithm (Equation
(\ref{Equation:Sato})). This derivation shows that we can avoid
the computation of $I(\mu)^{-1}$ by using the natural parameters
instead of the expectation parameters and the function $\theta(\mu)$.
\end{proof}

Theorem 1 simplifies the sdEM's updating equation to, 
\begin{equation}\label{Equation:sdEMFirstOrder}
\mu_{t+1} = \mu_{t} - \rho_t\frac{\partial
\bar{\ell}(y_t,x_t,\theta(\mu_{t}))}{\partial \theta}
\end{equation}
sdEM can be interpreted as a stochastic gradient descent
algorithm iterating over the \emph{expectation parameters} and
guided by the natural gradient in this Riemannian space.

\begin{algorithm}[!htbp]
\caption{Stochastic Discriminative EM (sdEM)} \label{alg:sdEM}
\begin{algorithmic}[1]
\REQUIRE  $D$ is randomly shuffled.

\STATE $\mu_{0}=\bar{\alpha}$; \hspace{15mm}(initialize
according to the prior)

\STATE $\theta_{0}=\theta(\mu_{0})$;

\STATE $t=0$;

\REPEAT

\FOR{ $i=1,\ldots,n$}

\STATE \textbf{E-Step}: \hspace{7mm} $\mu_{t+1} = \mu_{t} -
\frac{1}{(1+\lambda t)}\frac{\partial
\bar{\ell}(y_i,x_i,\theta_{t})}{\partial \theta};$

\vspace{3mm} \STATE Check-Step: \hspace{1mm}
$\mu_{t+1}=Check(\mu_{t+1},\mathcal{S})$;

\vspace{3mm} \STATE \textbf{M-Step}: \hspace{6mm}
    $\theta_{t+1}=\theta(\mu_{t+1});$

\STATE $t=t+1$;

\ENDFOR

\UNTIL convergence

\RETURN $\theta(\mu_{t})$;

\end{algorithmic}
\end{algorithm}
An alternative proof to Theorem \ref{Theorem:Sato} based on  more recent results on information geometry has been recently given in \cite{raskutti2013information}. The results of that work indicate that sdEM could also be interpreted as a mirror descent algorithm with a Bregman divergence as a proximitiy measure. It is beyond the scope of the paper to explore this relevant connection.

\subsection{CONVERGENCE OF sdEM}\label{Secton:sdEMConvergence}
In this section we do not attempt to give a formal proof of the
convergence of sdEM, since very careful technical arguments would be
needed for this purpose \cite{bottou1998online}. We simply go through
the main elements that define the convergence of sdEM as an stochastic approximation method \cite{kushner1997stochastic}. 

According to
Equation (\ref{Equation:sdEMSecondOrder}), sdEM can be seen as a
stochastic gradient descent method with the inverse of the Fisher
information matrix $I(\mu)^{-1}$ as a coefficient matrix
\cite{bottou1998online}. As we are dealing with exponential families,
these matrices are always positive-definite. Moreover, if the
gradient  $\partial\bar{\ell}(y,x,\theta)/\partial\theta$ can
be computed exactly (in Section \ref{Section:PartiallyObservable} we discuss what happens when
this is not possible), from Theorem \ref{Theorem:Sato}, we have
that it is an unbiased estimator of the natural gradient of
the $L(\theta(\mu))$ defined in Equation \ref{Equation:sdEMPiorLoss},
\begin{equation}\label{Equation:UnbiasedEstimator}
\e{\frac{\partial \bar{\ell}(y,x,\theta)}{\partial \theta}|\pi} =
I(\mu)^{-1}\frac{\partial L(\theta(\mu))}{\partial\mu}
\end{equation}


\begin{table*}[!htbp]
\caption{\label{Table:Losses} sdEM updating equations for fully observed data (Section \ref{Section:Losses}) .}
\begin{center}
\begin{tabular}{lr}
\hline
\hline
\textbf{Loss} &  \textbf{sdEM equation}\\
\hline \\
NLL & $\mu_{t+1} = (1-\rho_t(1+\frac{\nu}{n}))\mu_{t} +
\rho_t\left(s(y_t,x_t)+ \frac{1}{n}\bar{\alpha} \right)$
\\\\\hline\\
NCLL &  $\mu_{t+1} = (1-\rho_t\frac{\nu}{n})\mu_{t} +\rho_t \left(
s(y_t,x_t) - E_y[s(y,x_t)|\theta(\mu_{t})]  +
\frac{1}{n}\bar{\alpha} \right)$
\\\\\hline\\
Hinge & $\mu_{t+1} =  (1-\rho_t\frac{\nu}{n})\mu_{t}
+\rho_t\left\{\begin{array}{c l}
\frac{1}{n}\bar{\alpha} & \textbf{if}\hspace{2mm}\ln\frac{p(y_t,x_t|\theta)}{p(\bar{y}_t,x_t|\theta)}>1\\
s(y_t,x_t) - s(\bar{y}_t,x_t) + \frac{1}{n}\bar{\alpha}
&\textbf{otherwise}
\end{array}\right.$
\\\\
& \multicolumn{1}{r}{where $\bar{y}_t = arg\max_{y\neq y_t}
p(y,x_t|\theta)$}
\\\\\hline
\end{tabular}
\end{center}
\end{table*}

However, one key difference in terms of convergence between online EM and sdEM can be seen in
Equation (\ref{Equation:SatoIteration}): $\mu_{t+1}$ is a convex
combination between $\mu_t$ and the expected sufficient
statistics. Then, $\mu_{t+1}\in\mathcal{S}$ during all the
iterations. As will be clear in the next section, we do not have this same
guarantee in sdEM, but we can take advantage of the log prior term of Equation (\ref{Equation:sdEMPiorLoss}) to avoid this problem. This term plays a dual role as both ``regularization'' term and log-barrier function \cite{wright1999numerical} i.e. a continuous
function whose value increases to infinity as the parameter
approaches the \emph{boundary of the feasible region} or the support of
$p(\theta(\mu)|\alpha)$ \footnote{The prior $p$ would need to be suitably chosen.}. Then, if the step sizes $\rho_t$ are
small enough (as happens near convergence), sdEM will always stays
in the feasible region $\mathcal{S}$, due to the effect of the log
prior term. The only problem is that, in the initial iterations, the
step sizes $\rho_t$ are large, so one iteration can jump out of
the boundary of $\mathcal{S}$.  The method to avoid that depends on the
particular model, but for the models examined in this work it seems to be 
a simple check in every iteration. For example, as we
will see in the experimental section when implementing a
multinomial Naive Bayes, we will check at every iteration that each
sufficient statistic or ``word count'' is always positive. If a
``word count'' is negative at some point, we will set it to a very small value.
As mentioned above, this does not hurt the convergence of sdEM
because in the limit this problem disappears due the effect of the
log-prior term.

The last ingredient required to assess the convergence of a
stochastic gradient descent method is to verify that the sequence
of step sizes satisfies: $\sum \rho_t = \infty, \hspace{3mm} \sum
\rho_t^2 < \infty$.

So, if the sequence $(\mu_{t})_{t\geq 0}$ converges, it will
probably converge to the global minimum $(\mu^\star$,
$\theta^\star=\theta(\mu^\star))$ if $L(\theta)$ is convex, or to
a local minimum if $L(\theta)$ is not convex \cite{bottou1998online}.



Finally, we give an algorithmic description of sdEM in Algorithm
\ref{alg:sdEM}. Following \cite{bottou2012stochastic}, we consider steps
sizes of the form $\rho_t=(1+\lambda t)^{-1}$, where $\lambda$ is a
positive scalar\footnote{Our experiments suggest that trying
$\lambda\in \{1,0.1,0.01,$ $0.001,\ldots\}$ suffices for obtaining a quick
convergence.}. As mentioned above, the ``Check-Step'' is introduced to guarantee that
$\mu_t$ is always in $\mathcal{S}$. Like the online EM algorithm \cite{sato2000convergence,cappe2009line}, Algorithm
\ref{alg:sdEM} resembles the classic \emph{expectation
maximization} algorithm \cite{dempster1977maximum} since, as we will see in the next
section, the gradient is computed using \emph{expected
sufficient statistics}. Assumption \ref{Assumption:ML} guarantees
that the maximization step can be performed efficiently. This step
differentiates sdEM from classic stochastic gradient descent
methods, where such a computation does not exist.

\subsection{DISCRIMINATIVE LOSS FUNCTIONS}\label{Section:Losses}
As we have seen so far, the derivation of sdEM is complete
except for the definition of the loss function. We will discuss
now how two well known \emph{discriminative loss functions} can be
used with this algorithm.

\begin{table*}[!htbp]
\caption{\label{Table:HiddenLosses} sdEM updating equations for partially observed data (Section \ref{Section:PartiallyObservable}) }
\begin{center}
\begin{tabular}{lr}
\hline
\hline
\textbf{Loss} &  \textbf{sdEM equation}\\
\hline \\
NLL & $\mu_{t+1} = (1-\rho_t(1+\frac{\nu}{n}))\mu_{t} +
\rho_t\left( E_z[s(y_t,z,x_t)|\theta(\mu_{t})]+
\frac{1}{n}\bar{\alpha} \right)$
\\\\\hline\\
NCLL & $\begin{array}{ccl} \mu_{t+1} &=&
(1-\rho_t\frac{\nu}{n})\mu_{t} +
\rho_t\left(E_z[s(y_t,z,x_t)|\theta(\mu_{t})] -
E_{yz}[s(y,z,x_t)|\theta(\mu_{t})] + \frac{1}{n}\bar{\alpha}
\right)
\end{array}$
\\\\\hline\\

Hinge & $\mu_{t+1} =  (1-\rho_t\frac{\nu}{n})\mu_{t}
+\rho_t\left\{\begin{array}{c l}
\frac{1}{n}\bar{\alpha} & \textbf{if}\hspace{2mm}\ln\frac{\int p(y_t,z,x_t|\theta)dz}{\int p(\bar{y}_t,z,x_t|\theta)dz}>1\\
\begin{array}{l} E_z[s(y_t,z,x_t)|\theta(\mu_{t})]\\-E_z[s(\bar{y}_t,z,x_t)|\theta(\mu_{t})]
+ \frac{1}{n}\bar{\alpha}\end{array} &\textbf{otherwise}
\end{array}\right.$\\\\
& \multicolumn{1}{r}{where $\bar{y}_t = arg\max_{y\neq y_t} \int
p(y,z,x_t|\theta)dz$}
\\\\
\hline
\end{tabular}
\end{center}
\end{table*}

\subsubsection*{Negative Conditional Log-likelihood (NCLL)}\label{Section:NCLL}

As mentioned above, this loss function is defined as follows:
\begin{eqnarray*}
\ell_{CL}(y_t,x_t,\theta) &=&-\ln p(y_t,x_t|\theta) + \ln \int
p(y,x_t|\theta)dy
\end{eqnarray*}
And its gradient is computed as
\begin{eqnarray*}
\frac{\partial
\ell_{CL}(y_t,x_t,\theta)}{\partial\theta}=-s(y_t,x_t) +
E_y[s(y,x_t)|\theta]
\end{eqnarray*}
\noindent where the sufficient statistic $s(y_t,x_t)$ comes from
the gradient of the $\ln p(y_t,x_t|\theta)$ term in the NCLL loss,
and the expected sufficient statistic $E_y[s(y,x_t)|\theta] = \int
s(y,x_t) p(y|x_t,\theta)dy$, comes from the gradient of the $\ln
\int p(y,x_t|\theta)dy$ term in the NCLL loss. As mentioned above,
the computation of the gradient is similar to the
\emph{expectation step} of the classic EM algorithm.

The iteration equation of sdEM for the NCLL loss is detailed
in Table \ref{Table:Losses}. We note that in the case of
multi-class prediction problems the integrals of the updating
equation are replaced by sums over the different classes of the
class variable $Y$. We also show the updating equation for the negative
log-likelihood (NLL) loss for comparison purposes.

\subsubsection*{The Hinge loss}\label{Section:Hinge}
Unlike the previous loss which is valid for continuous and discrete (and vector-valued) predictions, this loss is only valid for binary or multi-class classification problems.

Margin-based loss functions have been extensively used and studied by the machine
learning community for binary and multi-class classification
problems \cite{bartlett2006convexity}. However, in our view, the application of
margin-based losses (different from the negative conditional
log-likelihood) for discriminative training of probabilistic
generative models is scarce and based on ad-hoc learning methods
which, in general, are quite sophisticated \cite{PernkopfWT12}. In this section, we
discuss how sdEM can be used to minimize the empirical risk of one
of the most used margin-based losses, the Hinge loss, in binary
and multi-class classification problems. But, firstly, we discuss
how Hinge loss can be defined for probabilistic generative models.


We build on LeCun et al.'s ideas \cite{lecun2006tutorial} about energy-based learning
for prediction problems. 
LeCun et al. \cite{lecun2006tutorial}  define the Hinge loss for
energy-based models as follows,
\begin{equation*}
\max(0,1-(E(\bar{y}_t,x_t,w) - E(y_t,x_t,w))
\end{equation*}
\noindent where $E(\cdot)$ is the energy function parameterized by
a parameter vector $w$, $E(y_t,x_t,w)$ is the energy associated to
the correct answer $y_t$ and $E(\bar{y}_t,x_t,w)$ is the energy
associated to the most offending incorrect answer, $\bar{y}_t =
arg\min_{y\neq y_t} E(y,x_t,w)$. Predictions $y^\star$ are made using $y^\star=arg\min_{y}
E(y,x_t,w^\star)$ when the parameter $w^\star$ that minimizes the empirical risk is found.


In our learning settings we consider the minus logarithm of
the joint probability, $-\ln p(y_t,x_t|\theta)$, as an energy function.
In consequence, we define the hinge loss as follows
\begin{equation}
\ell_{hinge}(y_t,x_t,\theta) =
\max(0,1-\ln\frac{p(y_t,x_t|\theta)}{p(\bar{y}_t,x_t|\theta)})
\end{equation}
\noindent where $\bar{y}_t$ denotes here too the most offending
incorrect answer, $\bar{y}_t = arg\max_{y\neq y_t}
p(y,x_t|\theta)$.

The gradient of this loss function can be simply computed as
follows
\begin{eqnarray*}
\frac{\partial
\ell_{hinge}(y_t,x_t,\theta)}{\partial\theta}=\left\{\begin{array}{l}
0\hspace{20mm}\textbf{if}\hspace{2mm}\ln\frac{
p(y_t,x_t|\theta)}{p(\bar{y}_t,x_t|\theta)}>1\\\\
-s(y_t,x_t) + s(\bar{y}_t,x_t) \hspace{2mm} \textbf{otherwise}
\end{array}\right.
\end{eqnarray*}
\noindent and the iteration equation for minimizing the empirical risk of the Hinge loss is also given in Table
\ref{Table:Losses}.


\begin{figure*}[!htbp]
\caption{\label{Figure:ToyExample} Toy example (Section \ref{Section:ToyExample}). The  result using the NLL loss (i.e. MLE estimation)
is plotted with dashed lines which represent
the densities $p(y=k)N(x,\mu^{(k)},\sigma^{(k)})$ for both classes
(i.e. when the red line is higher than the blue line we predict
the red class and vice versa). The estimated prediction accuracy of
the MLE model is 78.6\%. Solid lines represent the same
estimation but using the NCLL and the Hinge loss. Their estimated
prediction accuracies are 90.4\% and 90.6\%, respectively.}
\begin{center}
\begin{tabular}{cc}
\epsfig{file=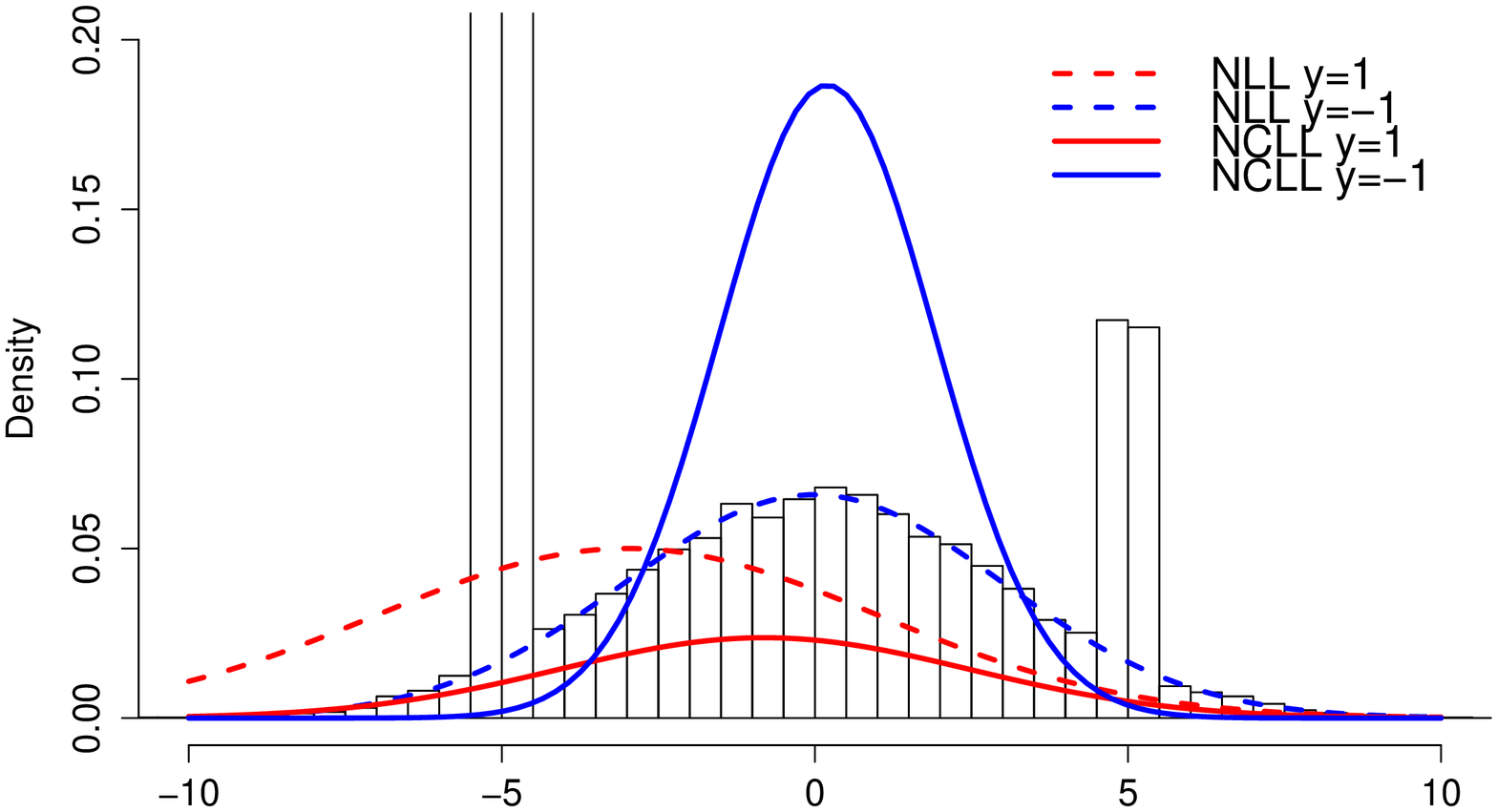,scale=0.35} &
\epsfig{file=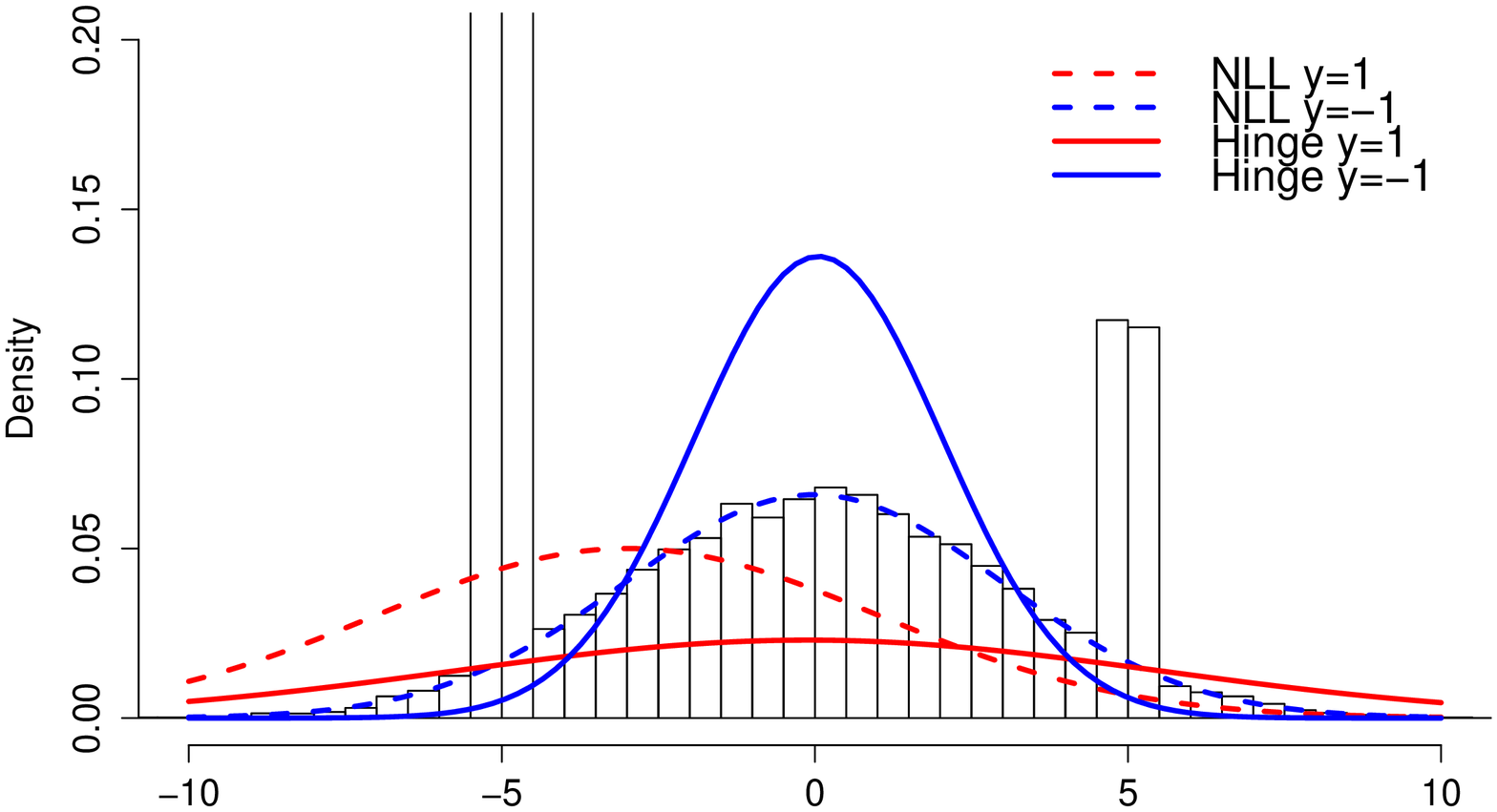,scale=0.35}\\
(a) NCLL Loss & (b) Hinge Loss\\
\end{tabular}
\end{center}
\end{figure*}

\subsection{PARTIALLY OBSERVABLE DATA}\label{Section:PartiallyObservable}


The generalization of sdEM to partially observable data is
straightforward. We denote by $Z$ the vector of non-observable
variables. sdEM will
handle statistical models which define a probability distribution
over $(y,z,x)$ which belongs to
the exponential family (Assumption \ref{Assumption:NEF}).
Assumption \ref{Assumption:ConjugatePrior} and \ref{Assumption:ML}
remain unaltered.

The tuple $(y,z,x)$ will denote the complete event or complete
data, while the tuple $(y,x)$ is the observed event or the observed
data. So we assume that our given data set
$D$ with $n$ observations is expressed as
$\{(y_1,x_1),\ldots,(y_n,x_n)\}$. So sdEM's Equation
(\ref{Equation:sdEMSecondOrder}) and
(\ref{Equation:sdEMFirstOrder}) are the same, with the only
difference that the natural gradient is now defined using the
inverse of the Fisher information matrix for the statistical model
$p(y,z,x|\theta(\mu))$. The same happens for Theorem
\ref{Theorem:Sato}.

The NCLL loss and the Hinge loss are equally defined as in Section \ref{Section:Losses}, with the
only difference that the computation of $p(y_t,x_t|\theta)$ and $p(x_t|\theta)$
requires marginalization over $z$, $p(y_t,x_t|\theta) = \int
p(y_t,z,x_t|\theta)dz$, $p(x_t|\theta) = \int
p(y,z,x_t|\theta)dydz$. The updating equations for sdEM under
partially observed data for the NCLL and Hinge loss are detailed
in Table \ref{Table:HiddenLosses}. New expected sufficient statistics
need to be computed, $E_{z}[s(y_t,z,x_t)|\theta] = \int
s(y_t,z,x_t)p(z|y_t,x_t,\theta)dz$ and $E_{yz}[s(y,z,x_t)|\theta]$
$= \int s(y,z,x_t)p(y,z|x_t,\theta)dydz$. As previously,  we also show 
the updating equation for the negative
log-likelihood (NLL) loss for comparison purposes.

\subsection{sdEM AND APPROXIMATE INFERENCE}\label{Section:missing}
For many interesting models \cite{blei2003latent}, the computation of the expected
sufficient statistics in the iteration equations shown in Table
\ref{Table:Losses} and \ref{Table:HiddenLosses} cannot be
computed in closed form. This is not a problem as far as we can
define \emph{unbiased estimators} for these expected sufficient
statistics, since the equality of Equation
(\ref{Equation:UnbiasedEstimator}) still holds. As it will be shown
in the next section, we use sdEM to discriminatively train
\emph{latent Dirichlet allocation} (LDA) models \cite{blei2003latent}. Similarly to \cite{RohdeCappe2011},
for this purpose we employ collapsed Gibbs sampling to compute the expected sufficient
statistics, $E_z[s(y_t,z,x_t)|\theta]$, as it guarantees
that at convergence samples are i.i.d. according to
$p(z|y_t,x_t,\theta)$.

\section{EXPERMINTAL ANALYSIS}\label{Section:Experiments}

\subsection{TOY EXAMPLE}\label{Section:ToyExample}

We begin the experimental analysis of sdEM by learning a very
simple Gaussian naive Bayes model composed by a binary class variable
$Y$ and a single continuous predictor $X$. Hence, the conditional
density of the predictor given the class variable is assumed to be
normally distributed. The interesting part of this toy example is
that the training data is generated by a different model:
$\pi(y=-1)=0.5$, $\pi(x|y=-1)\sim N(0,3)$ and $\pi(x|y=1)\sim
0.8\cdot N(-5,0.1) + 0.2\cdot N(5,0.1)$. Figure \ref{Figure:ToyExample} shows the
histogram of the 30,000 samples generated from the $\pi$
distribution. The result is a mixture of 3 Gaussians, one in the
center with a high variance associated to $y=-1$ and two narrows
Gaussians on both sides associated to $y=1$.

sdEM can be used by considering 6 (non-minimal) sufficient
statistics: $N^{(-1)}$ and $N^{(1)}$ as ``counts'' associated to
both classes, respectively; $S^{(-1)}$ and $S^{(1)}$ as the
``sum'' of the $x$ values associated to classes $y=-1$ and $y=1$,
respectively; and $V^{(-1)}$ and $V^{(1)}$ as the ``sum of
squares'' of the $x$ values for each class. We also have five parameters which
are computed from the sufficient statistics as follows: Two for
the prior of class
$p(y=-1)=p^{(-1)}=N^{(-1)}/(N^{(-1)}+N^{(1)})$ and
$p^{(1)}=N^{(1)}/(N^{(-1)}+N^{(1)})$; and four for the two
Gaussians which define the conditional of $X$ given $Y$,
$\mu^{(-1)}=S^{(-1)}/N^{(-1)}$,
$\sigma^{(-1)}=\sqrt{V^{(-1)}/N^{(-1)} - (S^{(-1)}/N^{(-1)})^2}$,
and equally for $\mu^{(1)}$ and $\sigma^{(1)}$.

The sdEM's updating equations for the NCLL loss can be written as follows
\begin{eqnarray*}
N^{(k)}_{t+1} = N^{(k)}_{t} + \rho_t(I[y_t=k] -
p_t(k|x_t)) +\frac{\rho_t}{n}\\
S^{(k)}_{t+1} = (1-\frac{\rho_t}{n})S^{(k)}_{t} + \rho_t
x_t\left(I[y_t=k] -
p_t(k|x_t)\right)\\
V^{(k)}_{t+1} = (1-\frac{\rho_t}{n})V^{(k)}_{t} + \rho_t
x_t^2(I[y_t=k] -
p_t(k|x_t))+\frac{\rho_t}{n}\\
\end{eqnarray*}
\noindent where $k$ indexes both classes, $k\in\{-1,1\}$,
$I[\cdot]$ denotes the indicator function, $p_t(k|x_t)$ is an
abbreviation of $p(y=k|x_t,\theta_t)$, and $\theta_t$ is the
parameter vector computed from the sufficient statistics at the
$t$-th iteration.

\begin{figure}[!htbp]
\caption{\label{Figure:MNBTrain}  Convergence trade-off of the Hinge loss versus the NCLL loss and the perplexity for a multinomial naive Bayes model trained minimizing the Hinge loss using sdEM. Circle-lines, triangle-lines and
cross-lines correspond to the results with 20NewsGroup, Cade and Reuters-R52 datasets, respectively. }
\begin{center}
\begin{tabular}{c}
\epsfig{file=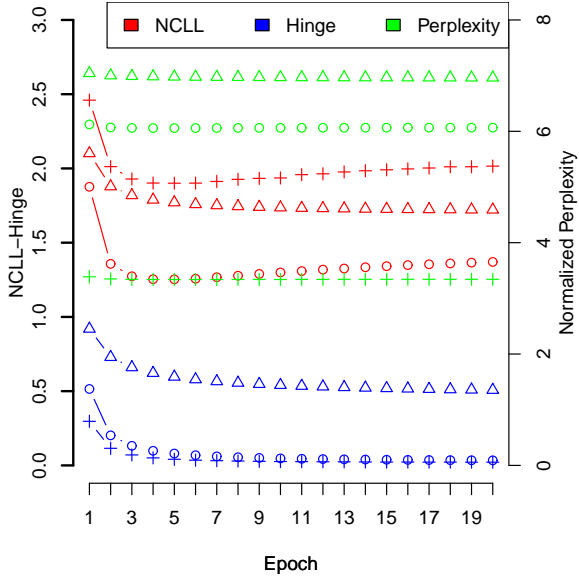,scale=0.6}\\
\end{tabular}
\end{center}
\end{figure}

\begin{figure}[!htbp]
\caption{\label{Figure:MNBTest} Convergence of the classification accuracy for a multinomial naive Bayes model trained minimizing the NCLL loss (NCLL-MNB) and the Hinge loss (Hinge-MNB) using sdEM. Red circle-lines, red triangle-lines and
red cross-lines correspond to the results of NCLL-MNB with 20NewsGroup, Cade and Reuters-R52 datasets, respectively. Same for Hinge-MNB.
The three blue and the three red solid lines detail the accuracy of logistic
regression and SVM, respectively. The three dashed black lines detail the
accuracy of plain MNB with a Laplace prior. }
\begin{center}
\begin{tabular}{c}
\epsfig{file=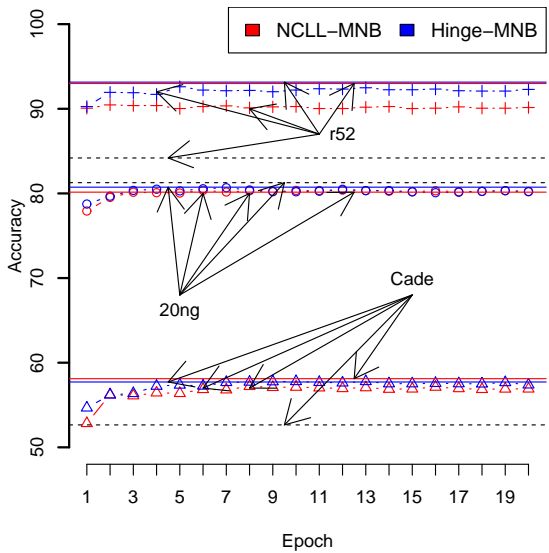,scale=0.57}\\
\end{tabular}
\end{center}
\end{figure}

Similarly, the sdEM's updating equations for the Hinge loss can be
written as follows,
\begin{eqnarray*}
N^{(k)}_{t+1} = N^{(k)}_{t} + k y_t \rho_t  I[\ln\frac{p_t(y_t|x_t)}{p_t(\bar{y_t}|x_t)}<1] +\frac{\rho_t}{n}\\
S^{(k)}_{t+1} = (1-\frac{\rho_t}{n})S^{(k)}_{t} + ky_t\rho_t
x_tI[\ln\frac{p_t(y_t|x_t)}{p_t(\bar{y_t}|x_t)}<1]\\
V^{(k)}_{t+1} = (1-\frac{\rho_t}{n})V^{(k)}_{t} + ky_t\rho_t
x_t^2I[\ln\frac{p_t(y_t|x_t)}{p_t(\bar{y_t}|x_t)}<1]+\frac{\rho_t}{n}\\
\end{eqnarray*}
\noindent where the product $ky_t$ is introduced in the updating
equations to define the sign of the sum, and the indicator
function $I[\cdot]$ defines when the hinge loss is null.


In the above set of equations we have considered as a conjugate prior for the
Gaussians a three parameter Normal-Gamma prior, $\nu=1$ and
$\bar{\alpha}_1=0$  for $S^{(k)}$ and $\bar{\alpha}_2=1$ for $V^{(k)}$
\cite[page 268]{bernardo2009bayesian}, and a Beta prior with $\nu=0$ and $\bar{\alpha}=1$ for $N^{(k)}$.
We note that these priors assign zero probability to
``extreme'' parameters $p^{(k)}=0$ (i.e. $N^{(k)}=0$) and  $\sigma^{(k)}=0$ (i.e. $V^{(k)}/N^{(k)} -
(S^{(k)}/N^{(k)})^2=0$).

Finally, the``Check-step'' (see Algorithm \ref{alg:sdEM}) performed before computing
$\theta_{t+1}$, and which guarantees that all sufficient statistics are
correct, is implemented as follows:
\begin{align*}
N^{(k)}_{t+1} = \max(N^{(k)}_{t+1},\frac{\rho_t}{n})\\
V^{(k)}_{t+1} =
\max(V^{(k)}_{t+1},\frac{(S^{(k)}_{t+1})^2}{N^{(k)}_{t+1}} +
\frac{\rho_t}{n})
\end{align*}
I.e., when the $N^{(k)}$ ``counts" are negative or too small or when the $V^{(k)}$ values lead to negative or null deviations $\sigma^{(k)}\leq 0$, they are fixed with the help of the prior term.

The result of this experiment is given in Figure
\ref{Figure:ToyExample} and clearly shows the different trade-offs
of both loss functions compared to \emph{maximum likelihood
estimation}. It is interesting to see how a generative model which 
does not match the underlying distribution is able to achieve a pretty high prediction accuracy
when trained with a discrimintaive loss function (using the sdEM algorithm).

\subsection{sdEM FOR TEXT CLASSIFICATION}

Next, we briefly show how sdEM can be used to discriminatively
train some generative models used for text classification,
such as multinomial naive Bayes and a similar classifier based on latent Dirichlet allocation models
\cite{blei2003latent}. Supplementary material with full details of these experiments and the Java code used in this evaluation can be download at: \url{http://sourceforge.net/projects/sdem/}

\begin{figure}[!htbp]
\caption{\label{Figure:AccuracyLDA}Convergence of the
classification accuracy of LDA classification models trained by sdEM using different loss functions (NLL, NCLL and Hinge) over 10 different
random initializations. The two dashed lines and the single solid line detail the
maximum, minimum and mean accuracy of sLDA, respectively, over 10
random initializations.}
\begin{center}
\epsfig{file=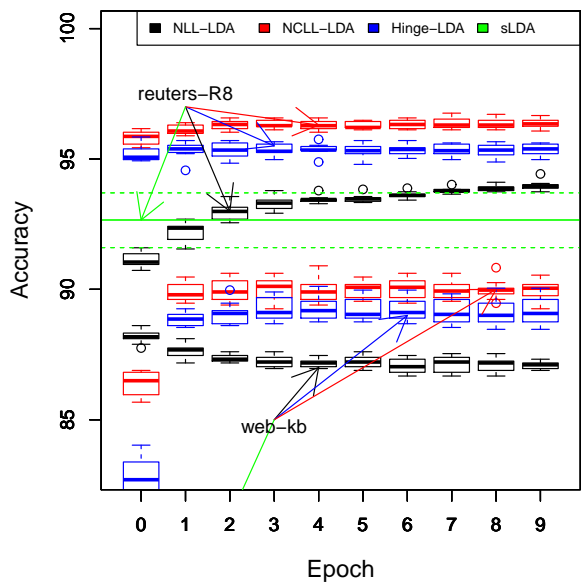,scale=0.6}
\end{center}
\end{figure}

\subsubsection*{Multinomial Naive Bayes (MNB)}
MNB assumes that words in documents with the same class or
label are distributed according to an independent multinomial
distribution.  sdEM can be easily applied to train this model. The sufficient
statistics are the ``prior class counts'' and the ``word counts''
for each class. The updating equations and the check step are
the same as those of $N^{(k)}_{t}$
in the previous toy example. Parameters of the MNB are computed
simply through normalization operations. Two different conjugate
Dirichlet distributions were considered: A ``Laplace prior'' where
$\bar{\alpha}_i=1$; and a ''Log prior'' where $\bar{\alpha}_i=$
``logarithm of the number of words in the corpus''. We only report
analysis for ``Laplace prior'' in the case of NCLL loss and for
``Log prior'' in the case of Hinge loss. Other combinations show
similar results, although NCLL was more sensitive to the chosen
prior.

We evaluate the application of sdEM to MNB with three well-known
multi-class text classification problems:  20Newsgroup (20
classes), Cade (12 classes) and Reuters21578-R52 (52 classes).
Data sets are stemmed. Full details about the data sets and the train/test data sets split used in this evaluation
can be found in \cite{2007:phd-Ana-Cardoso-Cachopo}. 


Figure \ref{Figure:MNBTrain} shows the convergence behavior of sdEM with
$\lambda=$1e-05 
when training
a MNB by minimizing the Hinge loss (Hinge-MNB). 
In this figure, we plot the evolution of the Hinge
loss but also the evolution of the NCLL loss and the normalized perplexity (i.e. the
perplexity measure \cite{blei2003latent} divided by the number of training
documents) at each epoch. We can see that there is a trade-off between the different
losses. E.g., Hinge-MNB decreases the Hinge loss (as expected) but tends
to increase the NCLL loss, while it only decreases perplexity at the very beginning.

Figure \ref{Figure:MNBTest} displays the evolution of the
classification accuracy of two MNBs trained minimizing the NCLL loss and the Hinge loss using sdEM. 
We compare them to: the standard MNB with a
``Laplace prior"; the L2-regularized Logistic
Regression; and the primal L2-regularized SVM. The two later methods were taken from the
Liblinear toolkit v.18 \cite{fan2008liblinear}. As can be seen,
sdEM is able to train simple MNB models with a performance very close
to that provided by highly optimized algorithms.

\subsubsection*{Latent Dirichlet Allocation (LDA)}
We briefly show the results of sdEM when discriminatively training LDA models. We define a classification model equal to MNB,
but where the documents of the same class are now modeled using an
independent LDA model. 
We implement this model by using, apart from the ``prior class counts",  the standard sufficient statistics of the LDA model, i.e. ``words per hidden topic counts", associated to each class label. Similarly to \cite{RohdeCappe2011}, we used an online Collapsed Gibbs sampling method to obtain, at convergence, unbiased estimates of the expected sufficient statistics (see Table \ref{Table:HiddenLosses}).

This evaluation was carried out using the standard train/test split of the Reuters21578-R8 (8 classes) and web-kb (4 classes) data sets
\cite{2007:phd-Ana-Cardoso-Cachopo}, under the same preprocessing than in the MNB's experiments.  Figure \ref{Figure:AccuracyLDA} shows the results of this
comparison using 2-topics LDA models trained with the NCLL loss (NCLL-LDA), the
Hinge loss (Hinge-LDA), and also the NLL loss (NLL-LDA)
following the updating equations of Table
\ref{Table:HiddenLosses}. We compared these results with those returned by supervised-LDA (sLDA) \cite{blei2007supervised}
using the same prior, but this time with 50 topics because less topics produced worse results. We see again how a simple generative model trained with sdEM outperforms much more sophisticated models.

\section{CONCLUSIONS}\label{Section:Conclusion}

We introduce a new learning algorithm for discriminative training
of generative models. This method is based on a novel view of the online EM algorithm as a
stochastic \emph{natural gradient} descent algorithm for minimizing general discriminative loss functions. It allows the
training of a wide set of generative models with or without latent
variables, because the resulting models are always generative. Moreover, sdEM is comparatively simpler and easier to implement (and debug) than other ad-hoc approaches.

\subsubsection*{Acknowledgments}

This work has been partially funded from the European Union's Seventh Framework Programme for research, technological development and demonstration under grant agreement no 619209 (AMIDST project).

\bibliographystyle{plain}
\bibliography{bibliography}

\begin{appendices}		
This supplementary material aims to extend, detail and complement the experimental evaluation of sdEM given in the main paper. The structure of this document is as follows. Section \ref{Section:MNB} details the experimental evaluation of the multinomial naive Bayes classifier and introduces new experiments comparing with the stochastic gradient descent algorithm \cite{bottou2010large}. The experimental evaluation of sdEM applied to latent Dirichlet allocation models is detailed and extended in Section \ref{Section:LDA}. Section \ref{Section:Code} points to the software repository where all the software code used in this experimental evaluation can be downloaded to reproduce all these results.

\section{Multinomial Naive Bayes for text classification}\label{Section:MNB}

\subsection*{Description of the algorithm}

As commented in the main paper, a multinomial Naive Bayes (MNB) classifier assumes that the words of the documents with the same class labels are distributed according to an independent multinomial probability distribution.  In this section we evaluate the use of sdEM to discriminatively train MNB models using the NCLL and the Hinge loss functions. In the first case, such a model would be related to a logistic regression model; while in the second case we will obtain a model directly related to a linear support vector machine classifier \cite{joachims1998text}. 

The general updating equations for this problem can be found in the main paper in Table 2. But a detailed pseudo-code description is now given in Algorithm \ref{alg:sdEM_MNB_NCLL}  for the NCLL loss and in Algorithm \ref{alg:sdEM_MNB_Hinge} for the Hinge loss.  In both cases, the sufficient statistics are the "prior class counts" stored in the matrix $C$ and the "word counts per class" stored in the matrix $N$.  Matrix $M$ is introduced to allow  efficient computations of the  posterior probability of the class variable given a document $d$, $p(Y|d,N,M,C,\gamma)$. How this posterior is computed is detailed in Algorithm \ref{alg:sdEM_MNB_Posterior}. In that way, the computational complexity of processing a label-document pair is linear in the number of words of the document. 
Finally, the function $Normalize(\cdot,\cdot)$ produces a multinomial probability by normalizing the vector of counts. The second argument contains the prior correction considered in this normalization, i.e. the value which is added to each single component of the count vector to avoid null probabilities, similar to what is done in Algorithm \ref{alg:sdEM_MNB_Posterior}.

In both algorithms, we consider a Dirichlet distribution prior for the multinomial distributions. As detailed in the header of these algorithms, two different priors are considered: prior $P1$, with Dirichlet's metaparameters $\alpha_k=1$; and prior $P2$ with $\alpha_k=\ln |W|$, where $|W|$ denotes the total number of different words in the corpus. In both cases, the prior assigns null probability to parameters lying in the "border" of the parameter space (i.e. when a null probability is assigned to some word). 

As commented in the "toy example" of  the main paper in Section 4.1, the parametrization that we chose for this Dirichlet prior makes that the $\nu$ parameter, arising in the exponential family form of this prior (see Assumption 2 of the main paper), be equal to null, $\nu=0$. This can be seen when expressing the Dirichlet distribution in the following exponential form:

\begin{eqnarray*}
Dir(\theta_1,\ldots,\theta_k;\alpha_1,\ldots,\alpha_K)=\frac{\prod_k \Gamma(\alpha_k)}{\Gamma(\sum_k\alpha_k)} \theta_1^{\alpha_1-1}\ldots \theta_K^{\alpha_K-1}\\
= exp\left(\sum_k(\alpha_k - 1)\ln\theta_k +\sum_k\ln\Gamma(\alpha_k) - \ln\Gamma(\sum_k\alpha_k) \right)\\
\end{eqnarray*}

The second and third terms inside the exponent in the above equation correspond to the log partition function $A_g(\alpha)$ of the prior. The first term correspond to the dot product between the sufficient statistics $(\ln\theta_1,\ldots,\ln\theta_K)$  and the natural parameters $(\alpha_1-1,\ldots,\alpha_k-1)$. As can be seen,  the $\nu$ parameter can be obviated in this definition,  i.e. $\nu=0$.

\begin{figure*}[!htbp]
\caption{\label{Figure:NCLLMNBTrain}  Convergence behavior of sdEM
applied to a multinomial naive Bayes model (NCLL-MNB) with different priors. Circle-lines, triangle-lines and
cross-lines correspond to the results with 20NewsGroup, Cade and Reuters-R52 datasets, respectively. In the third figure, the three solid lines detail the accuracy of logistic
regression for these three data sets. The tree dashed lines detail the
accuracy of plain MNB with P1 (MNB results with P2 are omitted because they are much worse).}
\begin{center}
\begin{tabular}{ccc}
\epsfig{file=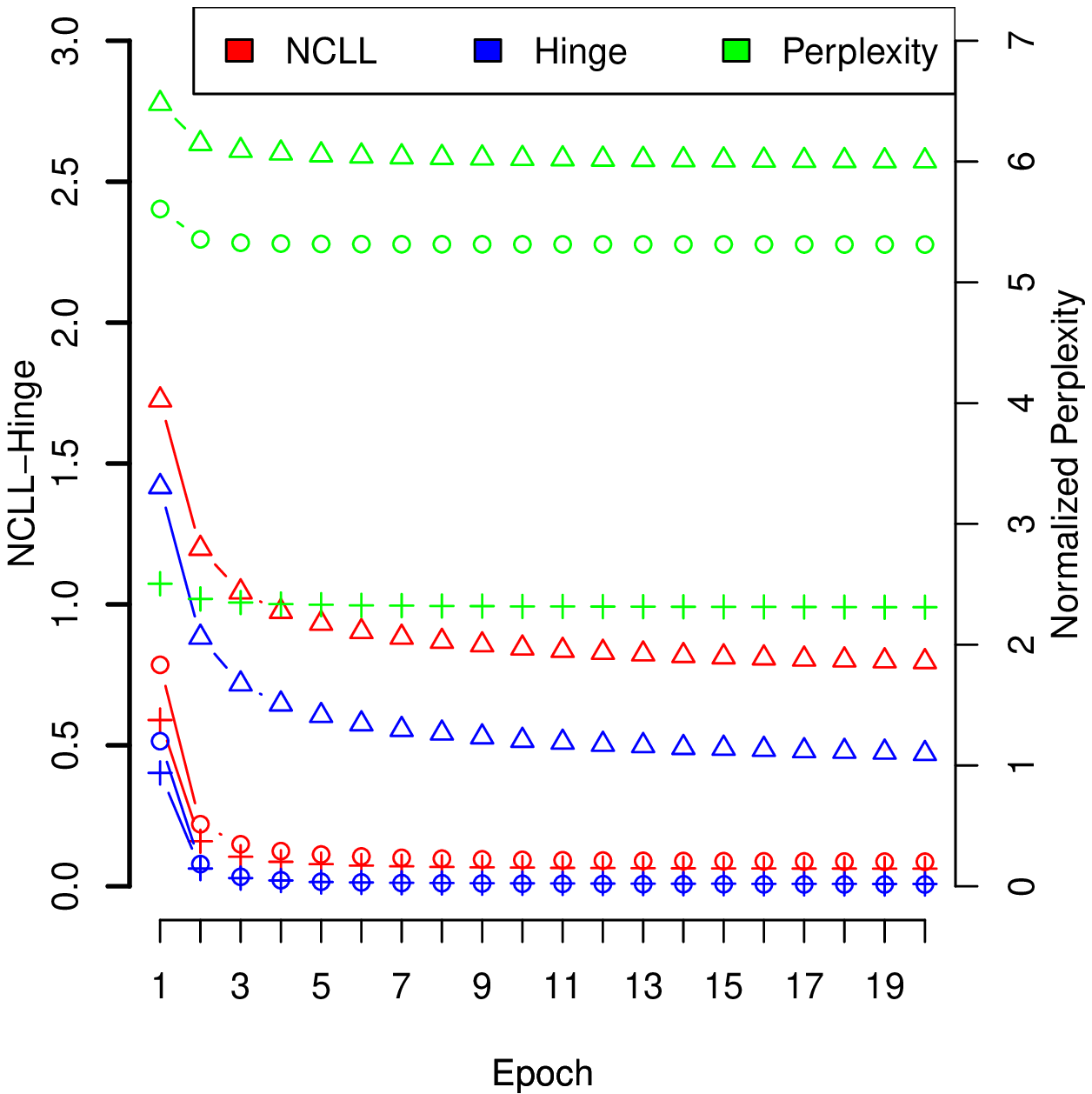,scale=0.43} &
\epsfig{file=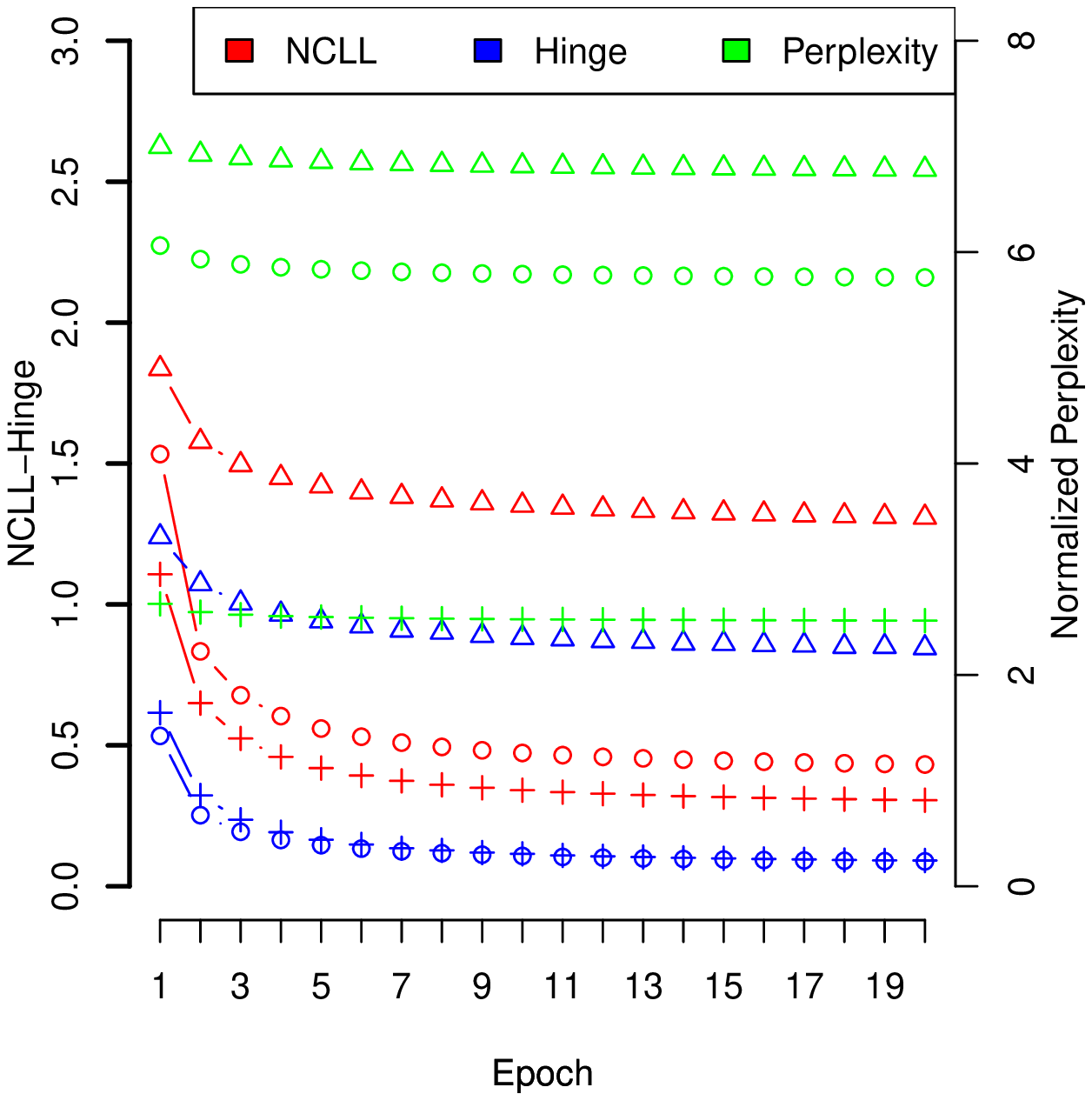,scale=0.43}&
\epsfig{file=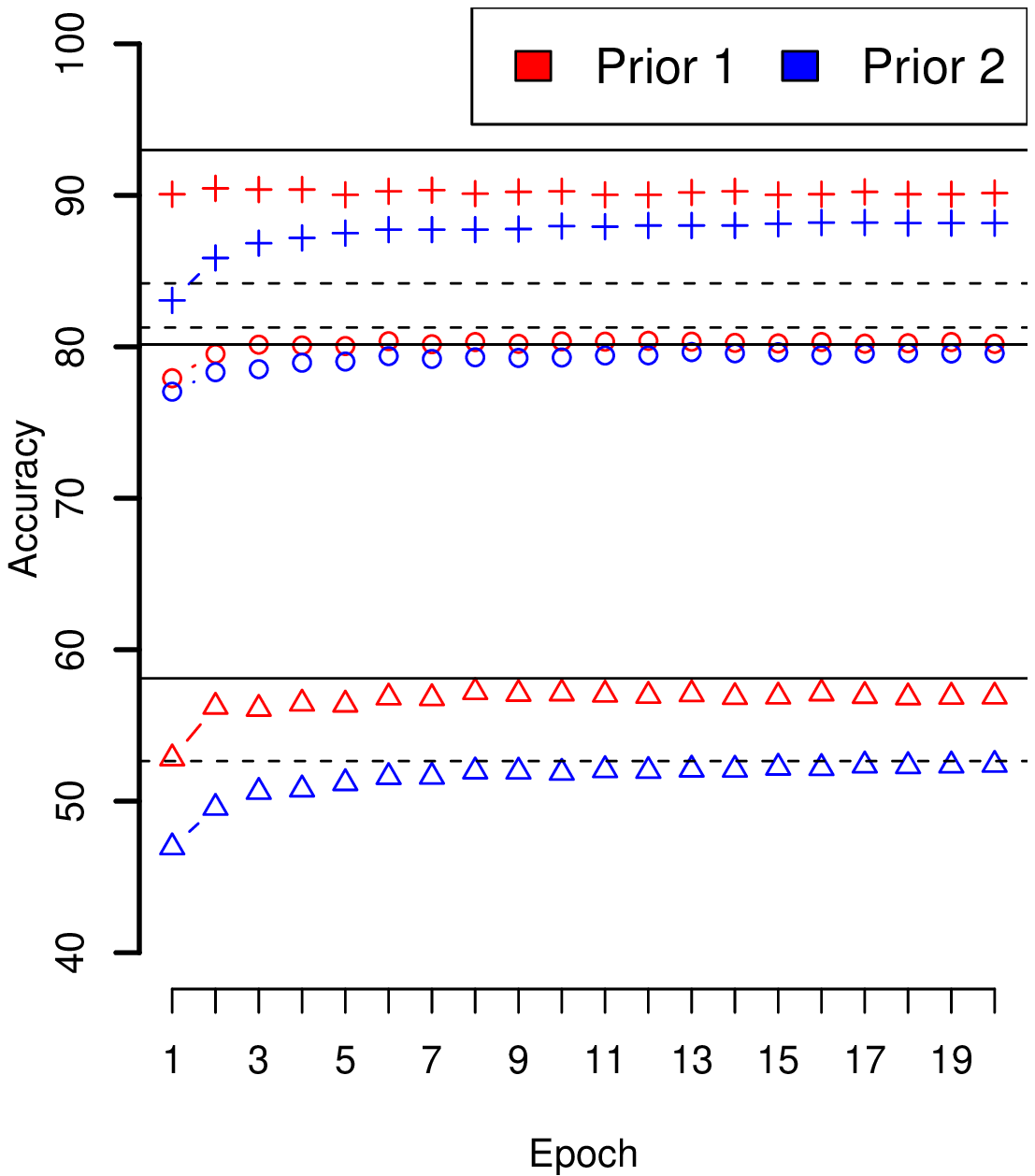,scale=0.43}\\
(a) Converg. of loss functs. with P1 & (b) Converg. of loss functs. with P2 & (c) Converg. of Accuracy\\
\end{tabular}
\end{center}
\end{figure*}

\begin{table}[!htbp]
\caption{\label{Table:Data} Data sets statistics. $|W|$ denotes the number of different words in the data set, $|Y|$ denotes the number of labels of the class variable,  $D_{train}$ and $D_{test}$ the number of documents in the training and the test set, respectively. These data sets can be downloaded, for example, from http://sourceforge.net/projects/sgmweka/.}
\begin{center}
\begin{tabular}{lcccc}
\hline
\textbf{Name} &  $|W|$  & $|Y|$ & $D_{train}$ & $D_{eval}$\\
\hline 
ACL-IMDB & 89527 & 2 & 47000 & 3000\\
Amazon12 & 86914 & 2 & 267875 & 100556\\
Cade  & 157483  & 12 & 27322 & 13661 \\
Reuters-R8   &  14575  &  8  & 2785  & 1396 \\
Reuters-R52  & 16145  & 52  & 5485  & 2189\\
WebKb  & 7287  & 4  & 6532  & 2568\\
20 News-group  & 54580  & 20  &  11293  & 7528\\
\hline
\end{tabular}
\end{center}
\end{table}

\subsection*{Experimental Evaluation}

As detailed in the main paper, we evaluate the application of sdEM to MNB with three well-known
multi-class text classification problems:  20Newsgroup, Cade and Reuters21578-R52. These data sets are stemmed. Full details about the data sets and the train/test data sets split used in this evaluation
can be found in \cite{2007:phd-Ana-Cardoso-Cachopo}. Although Table \ref{Table:Data} shows some of the main statistics of these data sets.

\begin{figure*}[!htbp]
\caption{\label{Figure:HingeMNBTrain} Convergence behavior of sdEM
applied to a multinomial naive Bayes model (Hinge-MNB) with different priors. Circle-lines, triangle-lines and
cross-lines correspond to the results with 20NewsGroup, Cade and Reuters-R52 datasets, respectively. In the third figure, the three solid lines detail the accuracy of SVM for these three data sets. The tree dashed lines detail the
accuracy of plain MNB with P1 (MNB results with P2 are omitted because they are much worse).}
\begin{center}
\begin{tabular}{ccc}
\epsfig{file=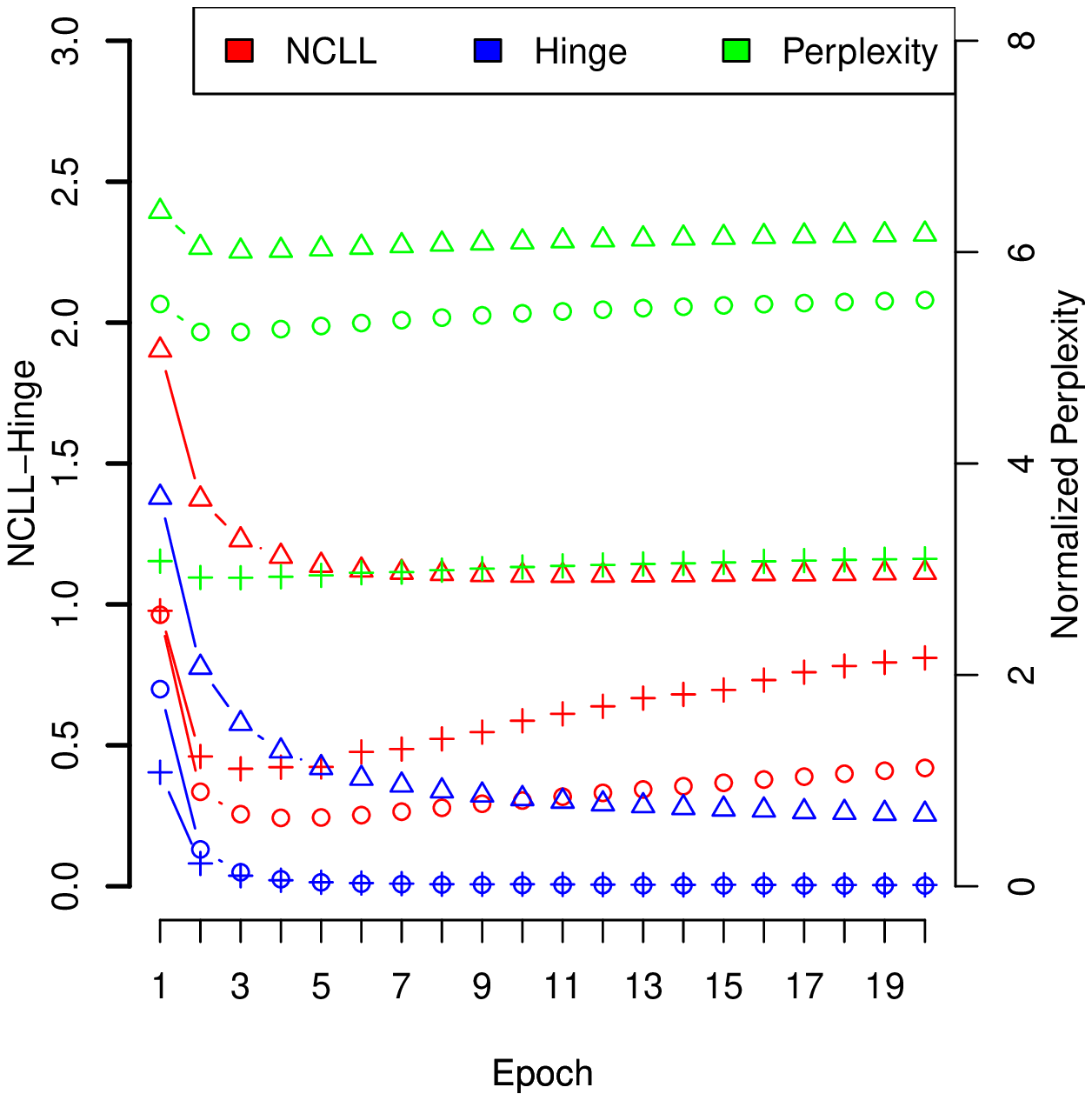,scale=0.43} &
\epsfig{file=./imgs/HingeMultinomialTrainP1.eps,scale=0.43}&
\epsfig{file=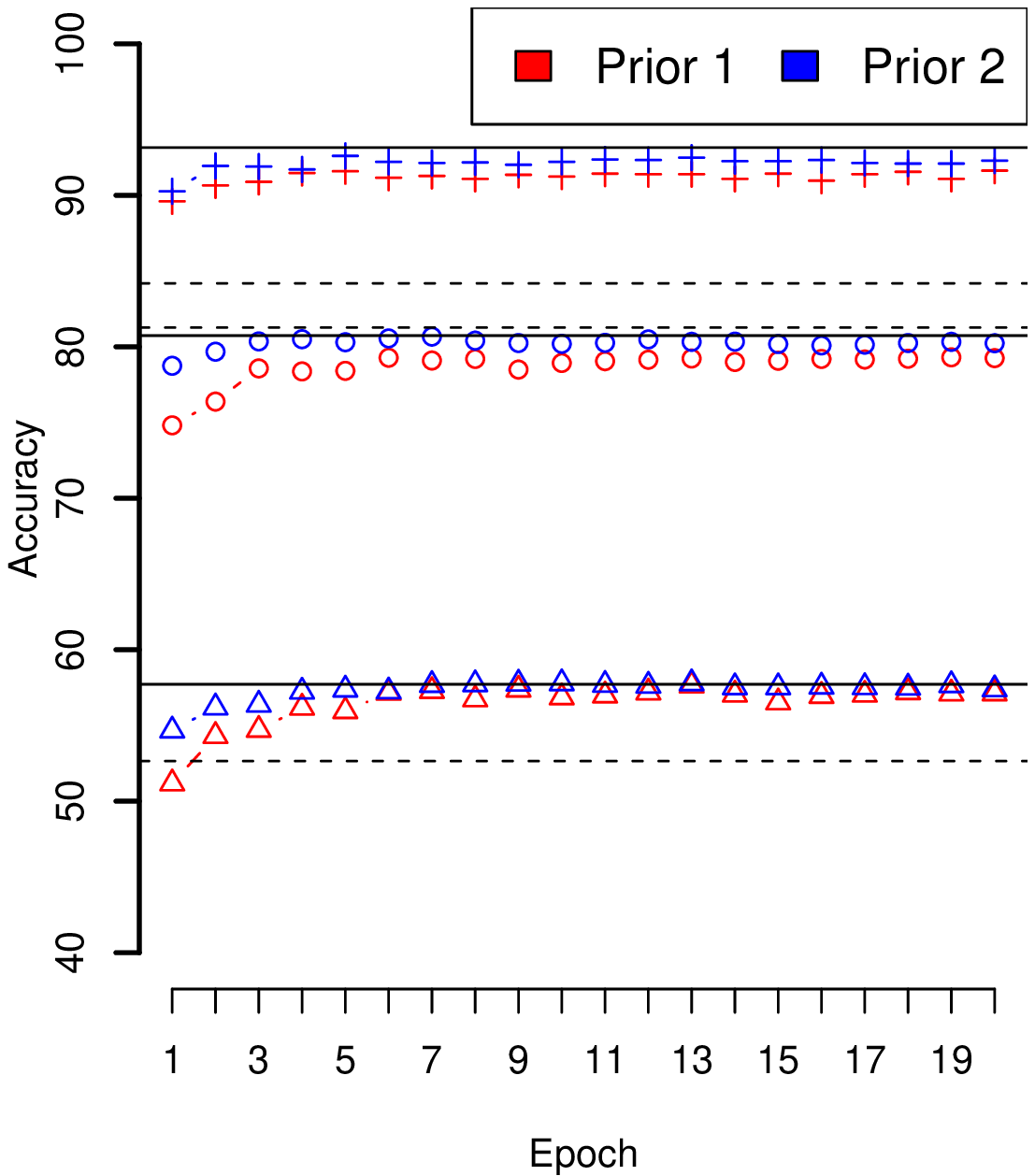,scale=0.43}\\
(a) Converg. of loss functs. with P1 & (b) Converg. of loss functs. with P2 & (c) Converg. of Accuracy\\
\end{tabular}
\end{center}
\end{figure*}

Figure \ref{Figure:NCLLMNBTrain} (a), Figure \ref{Figure:NCLLMNBTrain} (b),  Figure \ref{Figure:HingeMNBTrain} (a) and  Figure \ref{Figure:HingeMNBTrain} (b) show the convergence behavior of sdEM with $\lambda=$1e-05 \footnote{Other values yield similar results and
offer stable convergence, although at lower peace.} 
when training the MNB by minimizing the NCLL loss (NCLL-MNB) and by minimizing the Hinge loss (Hinge-MNB), respectively. 
In both cases, we plot the evolution of the NCLL
loss, the Hinge loss and the normalized perplexity (i.e. the
perplexity measure \cite{blei2003latent} divided by the number of training
documents) of the trained model at each epoch. We can see that there is a trade-off between the different
losses. For example, Hinge-MNB decreases the Hinge loss (as expected) but tends
to increase the NCLL loss, while it only decreases perplexity at the very beginning. This last trend is much stronger when considering the P1 prior. A similar behavior can be observed for NCLL-MNB, with the main difference that the NCLL loss is an upper bound of the Hinge loss, and then when NCLL-MNB minimizes the NCLL loss it also minimizes the Hinge loss. Here it can be also observed that the perplexity remains quite stable specially for P1.

Figure \ref{Figure:NCLLMNBTrain} (c) and Figure \ref{Figure:HingeMNBTrain} (c) displays the evolution of the
classification accuracy for the above models. We compare it to the standard MNB with a
``Laplace prior" \footnote{A ``Log prior'' was also evaluated but
reported much worse results.} 
 and with L2-regularized Logistic
Regression and primal L2-regularized SVM implemented in the
Liblinear toolkit v.18 \cite{fan2008liblinear}. For the case of the NCLL loss, the models seem to be more dependent of the chosen prior, specially for the Cade dataset. In any case, we can see that sdEM is able to train simple MNB models with a performance very close to that provided by highly optimized algorithms.

\begin{figure*}[!htbp]
\caption{\label{Figure:SGDNCLL} Convergence of the classification accuracy for NCLL-SGD and NCLL-MNB for different priors and different learning rates in the Amazon12 and ACL-IMDB data set. Solid lines detail the accuracy of the aforementioned Liblinear's logistic regression and dashed lines detail the accuracy of the plain MNB with the corresponding prior. The numbers of the bottom right legends correspond to different $\lambda$ values (see Section 3.2 of the main paper), which defines how the sdEM's learning rates $\rho_t$ decreases over time.}
\begin{center}
\begin{tabular}{ccc}
\multicolumn{3}{c}{Amazon12}\\
\epsfig{file=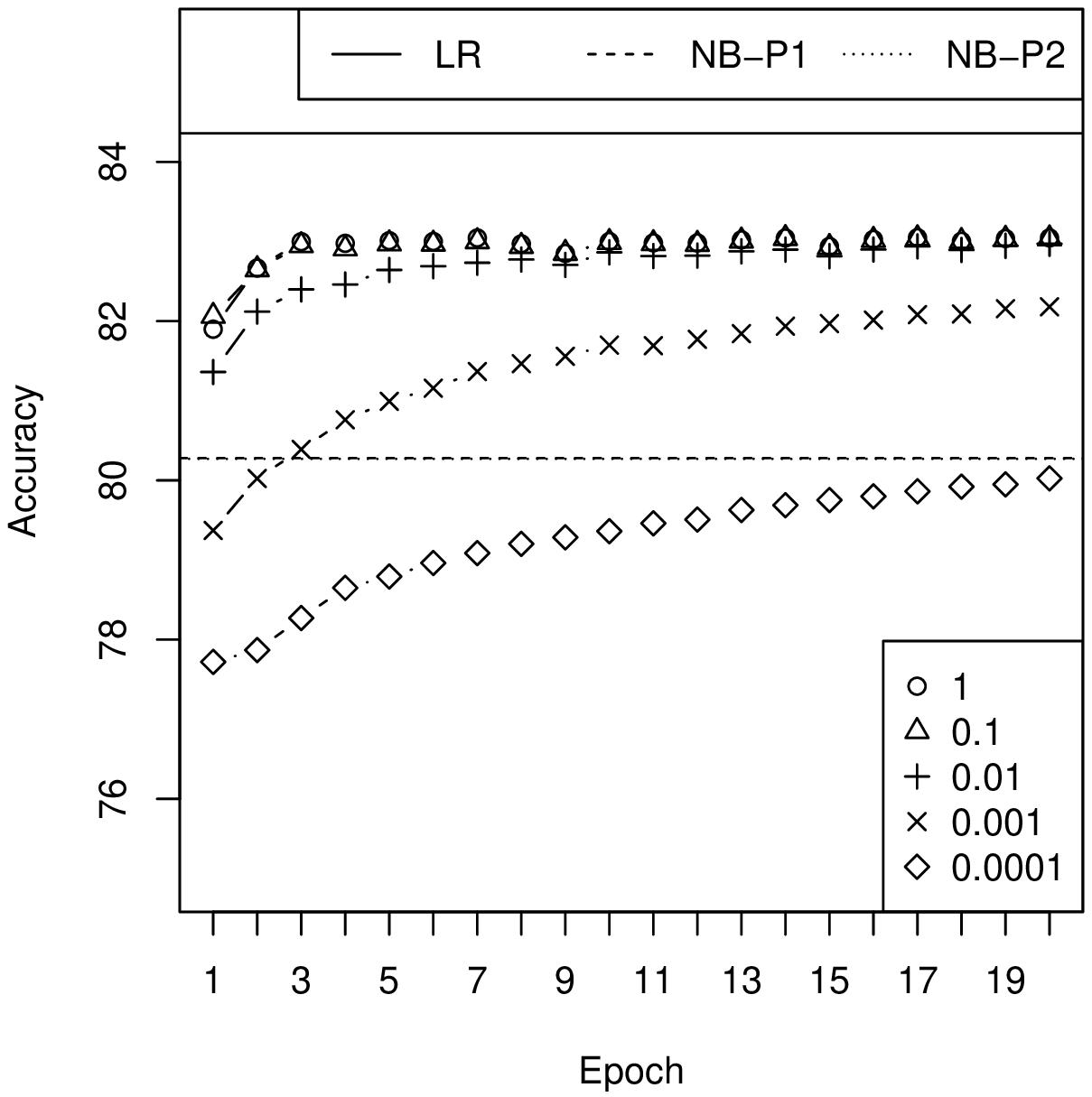,scale=0.43} &
\epsfig{file=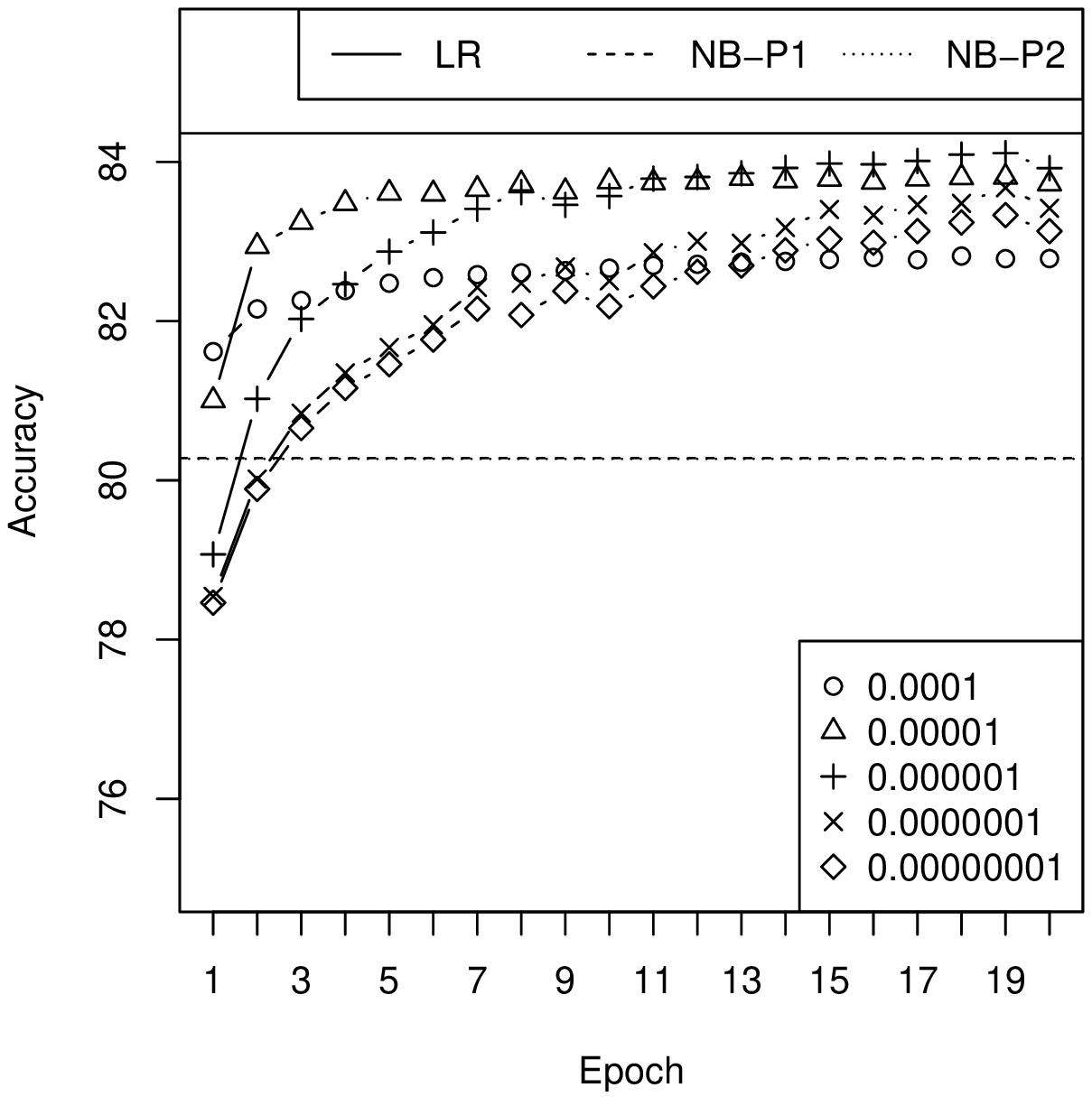,scale=0.43}&
\epsfig{file=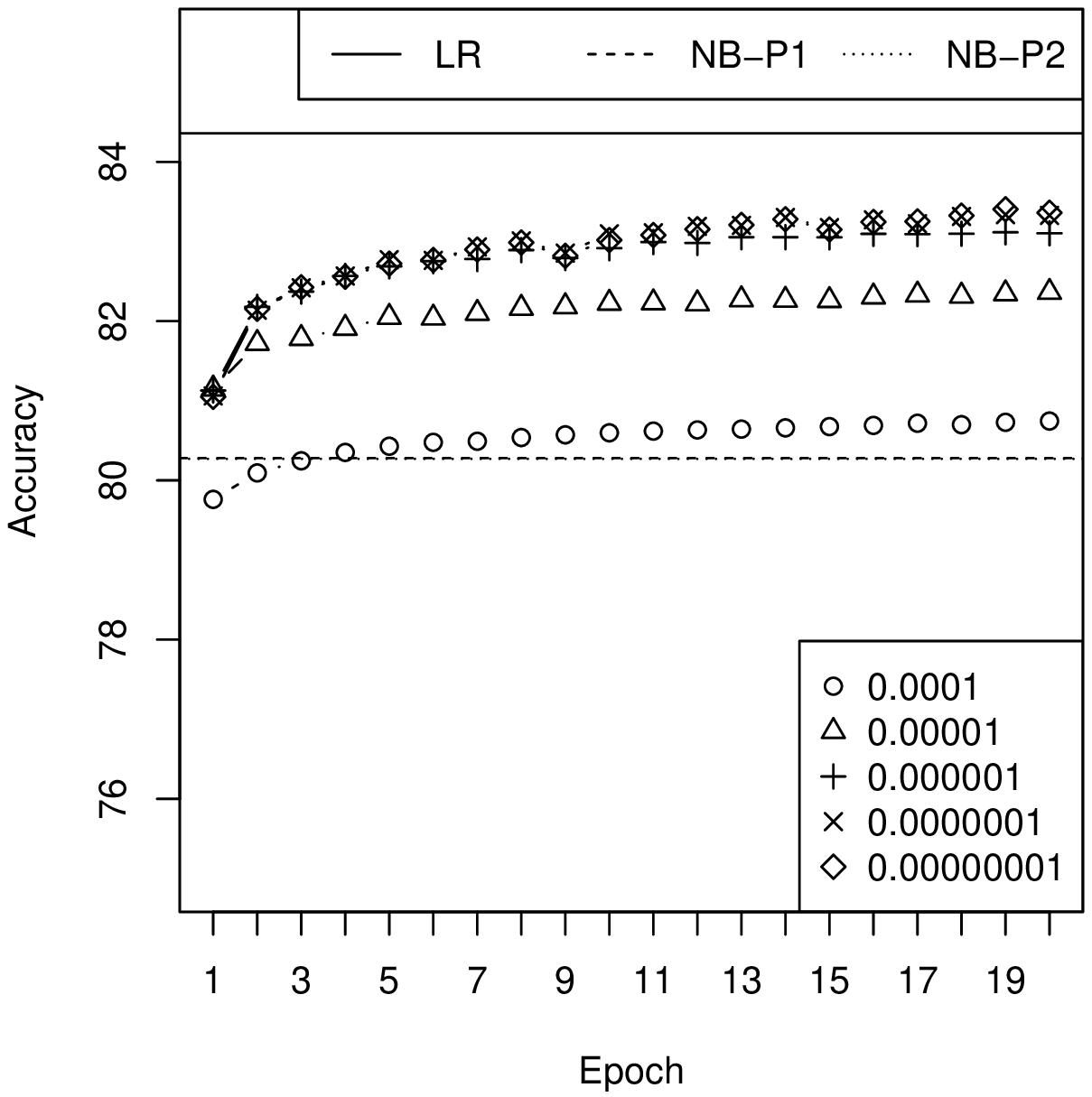,scale=0.43}\\
\\\multicolumn{3}{c}{ACL-IMDB}\\
\epsfig{file=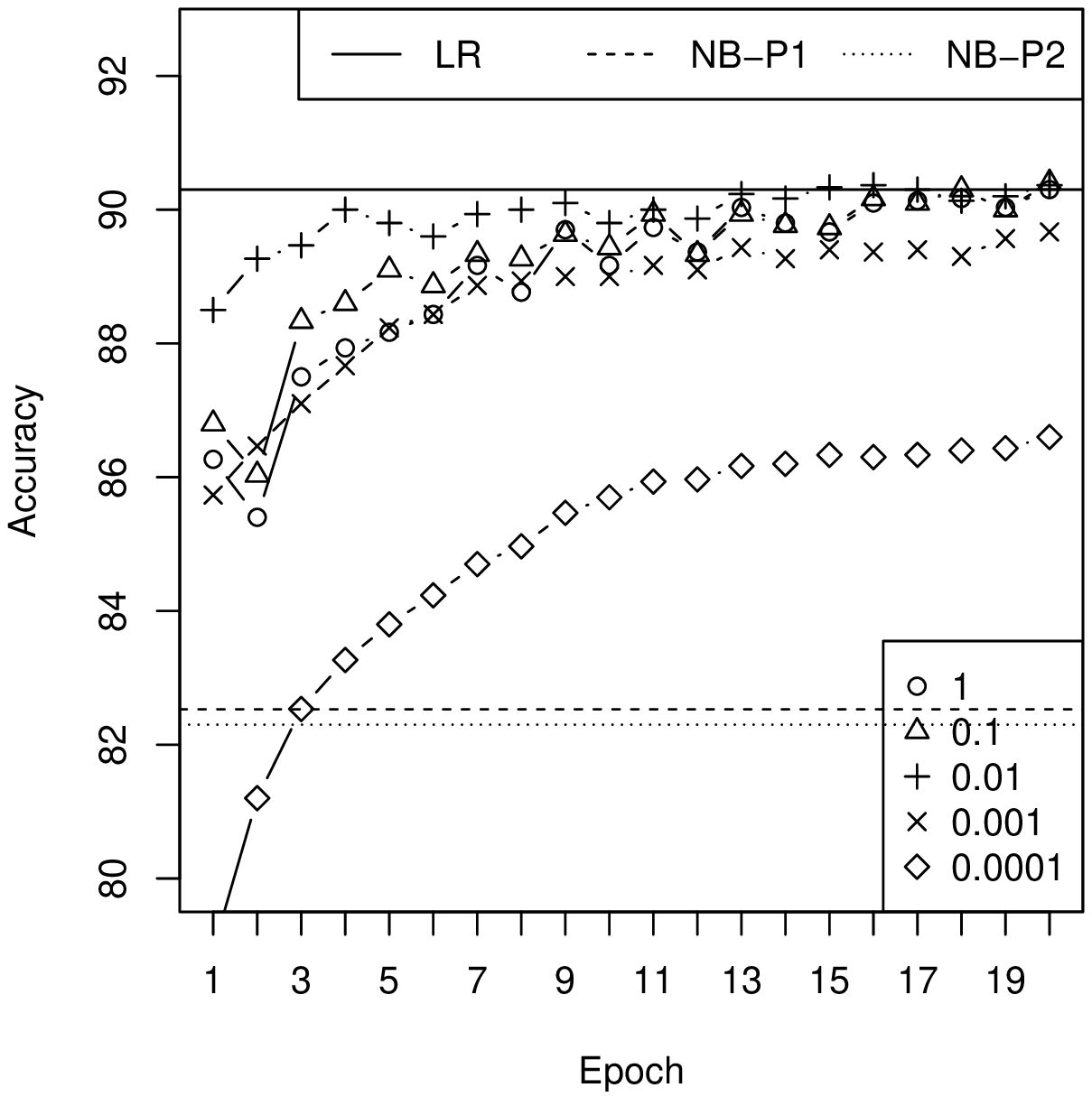,scale=0.43} &
\epsfig{file=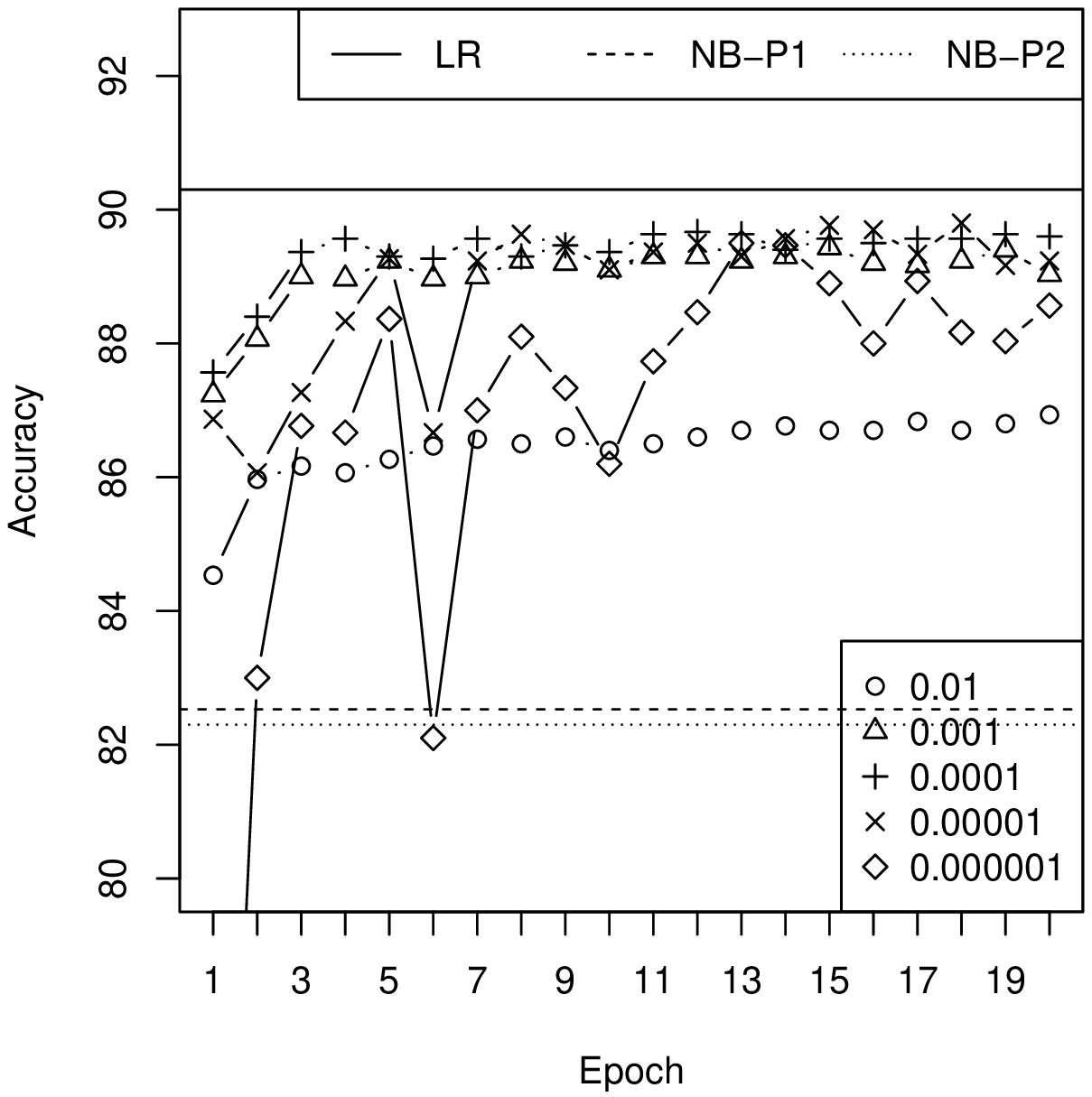,scale=0.43}&
\epsfig{file=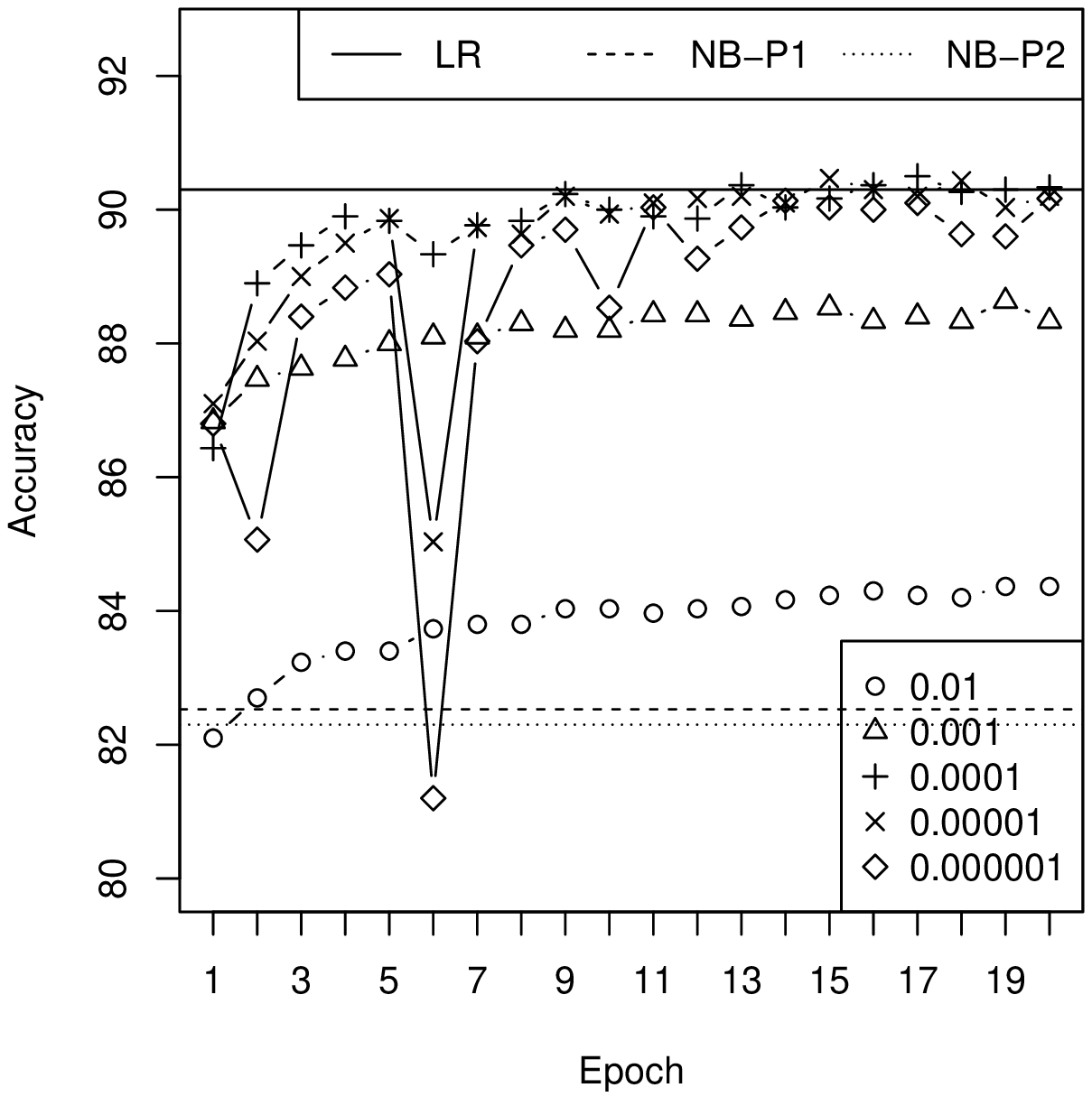,scale=0.43}\\
(a) NCLL-SGD & (b) NCLL-MNB with prior P1 & (c) NCLL-MNB with prior P2 \\
\end{tabular}
\end{center}
\end{figure*}

A new set of experiments is included in this analysis comparing the MNB models learnt with sdEM with the classic stochastic gradient descent (SGD) algorithm. This evaluation is made using the Amazon12 and ACL-IMDB data sets (whose main details can be found in Table \ref{Table:Data}). We choose these data sets because they are binary classification problems, which are very well defined problems for logistic regression and linear SVM models. How SGD is used to train this model can be seen in \cite{bottou2012stochastic}.

In this evaluation we simply plot the evolution of the classification accuracy of the SGD algorithm when training a linear classifier using the NCLL loss (NCLL-SGD) and the Hinge loss (Hinge-SGD) with a L2 regularization for different learning rates or decreasing steps $\rho_t$. SGD is implemented as detailed in \cite{bottou2012stochastic}, where the weight of the regularized  term is fixed to 1e-4.  As recommended in \cite{bottou2012stochastic}, learning rates $\rho_t$ for SGD are computed as follows: $\rho_t=\frac{\lambda}{1+\lambda\cdot 0.0001\cdot t}$. We also look at the evolution of the classification accuracy of NCLL-MNB and Hinge-MNB with different priors in these two data sets and using different learning rates. In both cases, the plotted learning rates $\rho_t$ are selected by using different $\lambda$ values of the form $\lambda\in \{1,0.1,0.01,0.001,0.0001,0.00001,\ldots\}$.  These results are shown in  Figures \ref{Figure:SGDNCLL} and \ref{Figure:SGDHinge}. In each case, we consider the 5 consecutive $\lambda$ values with the quickest convergence speed.

\begin{figure*}[!htbp]
\caption{\label{Figure:SGDHinge} Convergence of the classification accuracy for Hinge-SGD and Hinge-MNB and different learning rates for different priors in the Amazon12 and ACL-IMDB data set. Solid lines detail the accuracy of the aforementioned Liblinear's SVM classifier and dashed lines detail the accuracy of the plain MNB with the corresponding prior. The numbers of the bottom right legends correspond to different $\lambda$ values (see Section 3.2 of the main paper), which defines how the sdEM's learning rates $\rho_t$ decreases over time.}
\begin{center}
\begin{tabular}{ccc}
\multicolumn{3}{c}{Amazon12}\\
\epsfig{file=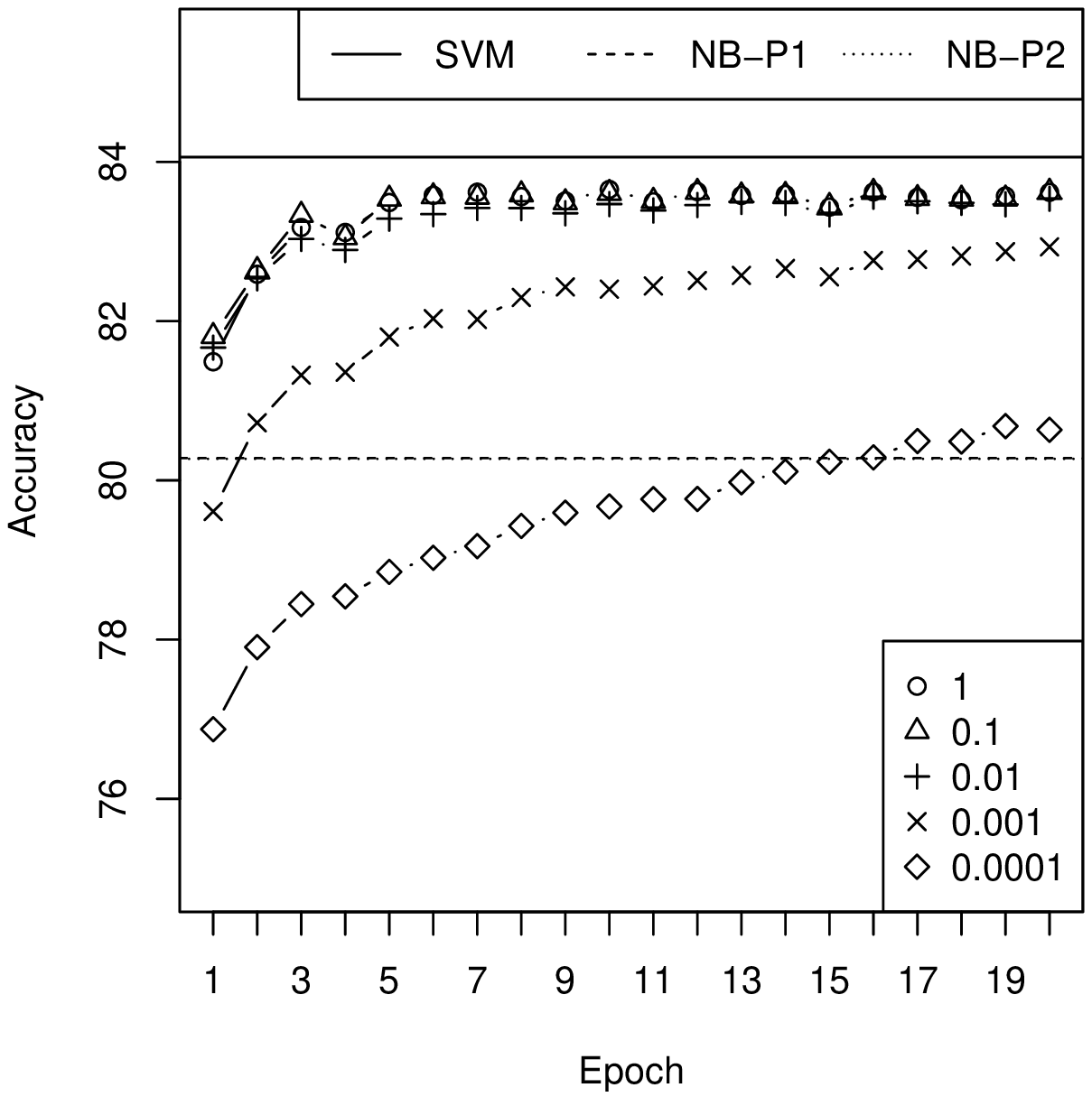,scale=0.43} &
\epsfig{file=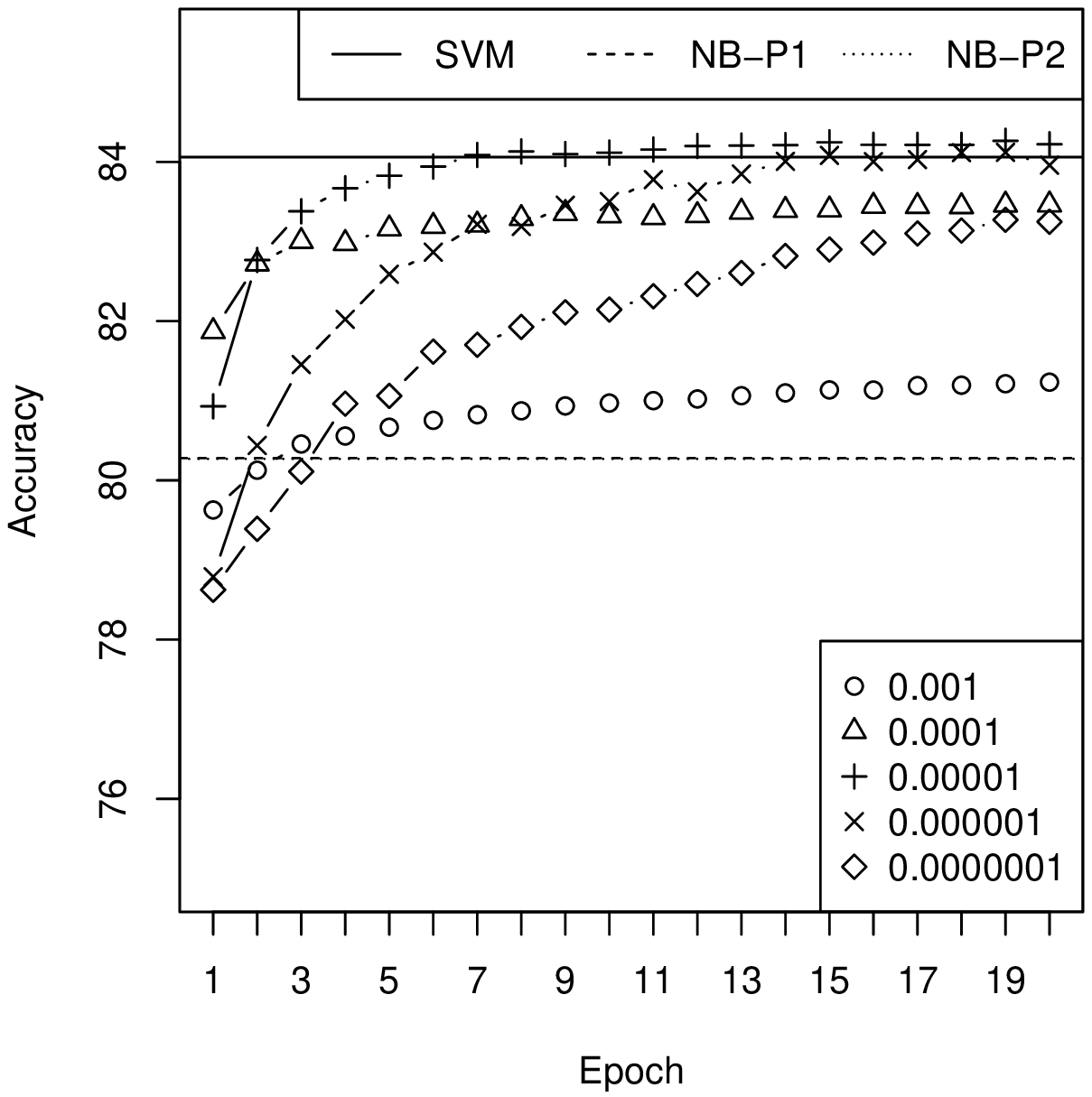,scale=0.43}&
\epsfig{file=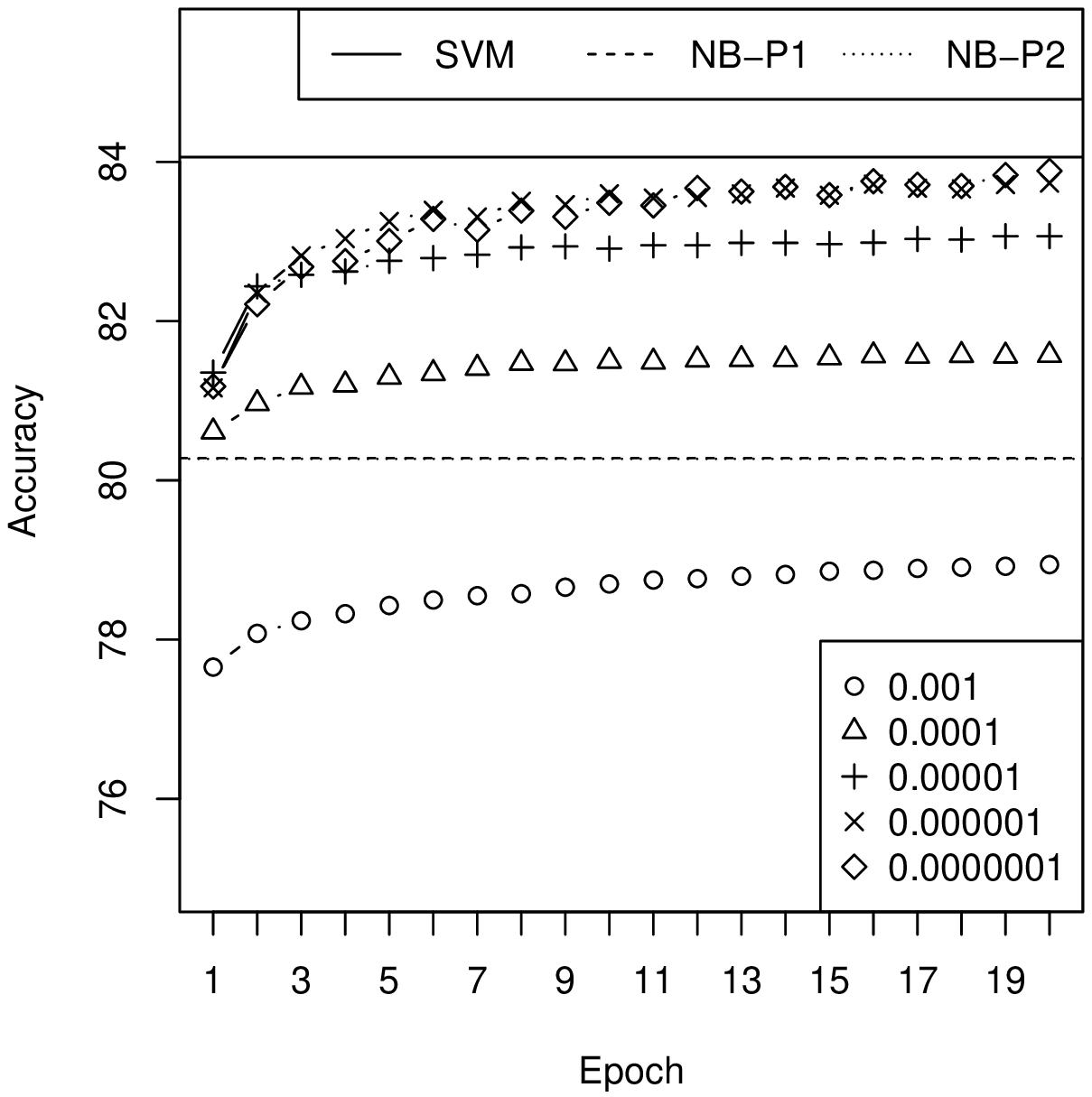,scale=0.43}\\
\\\multicolumn{3}{c}{ACL-IMDB}\\
\epsfig{file=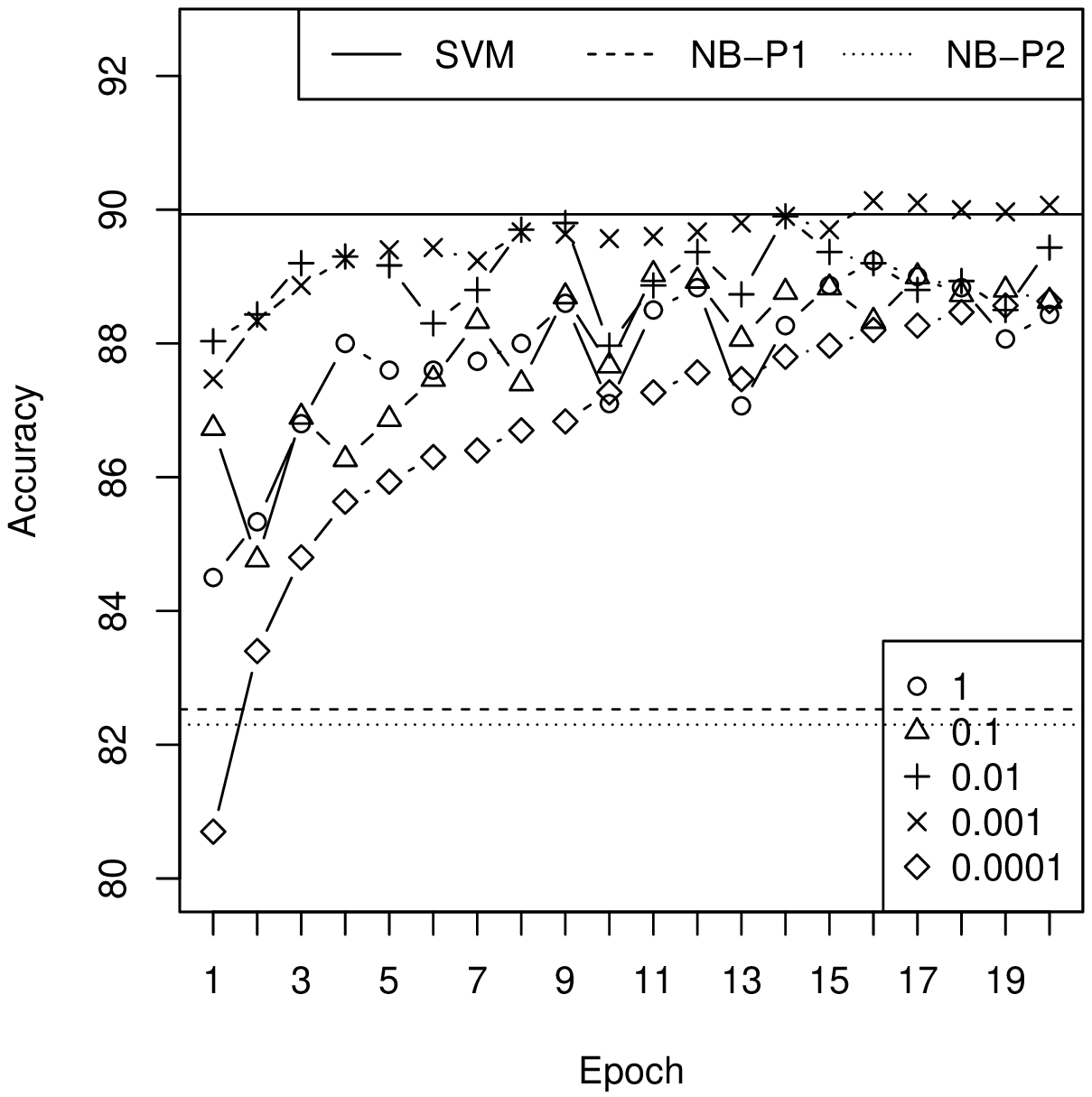,scale=0.43} &
\epsfig{file=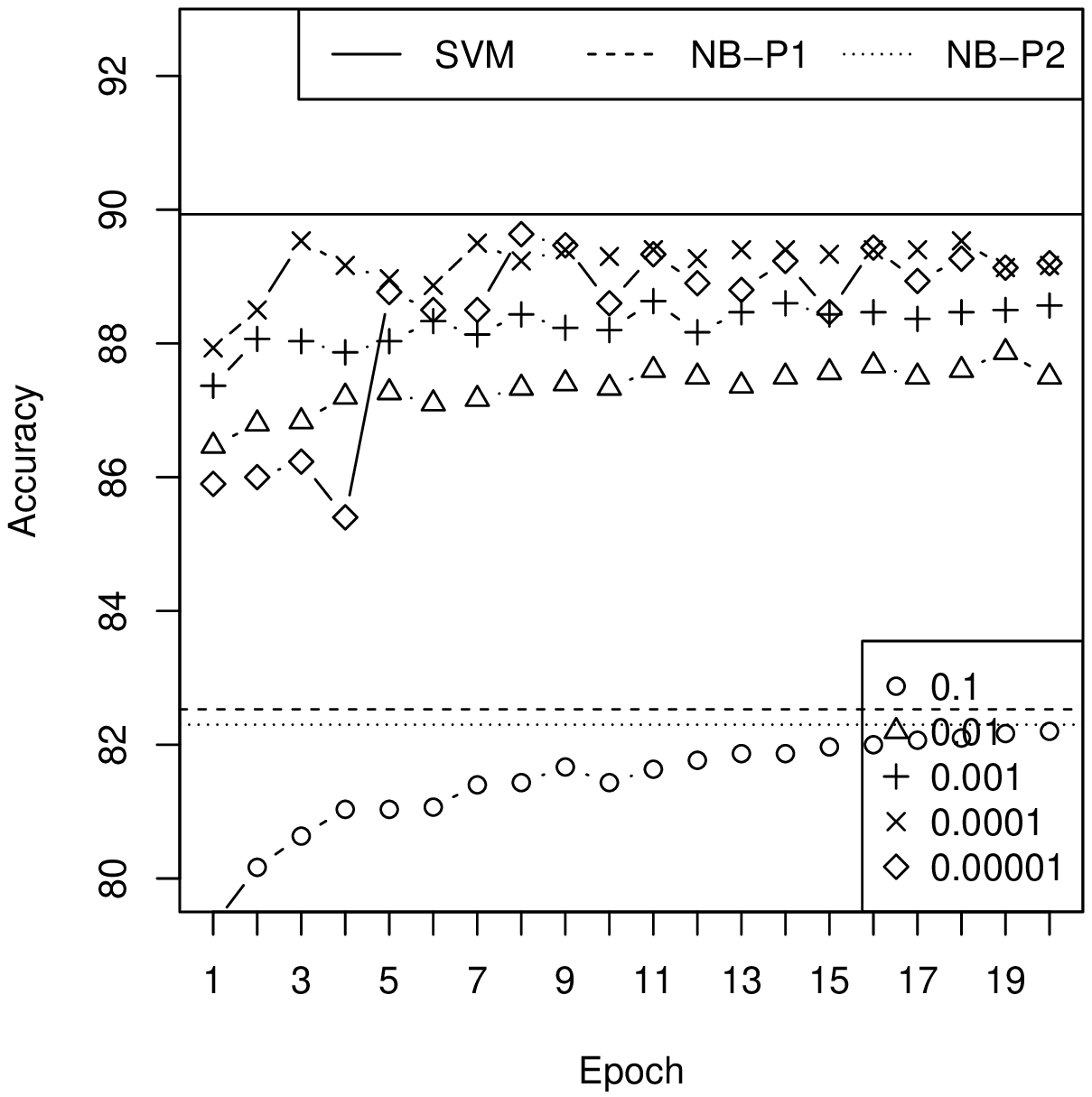,scale=0.43}&
\epsfig{file=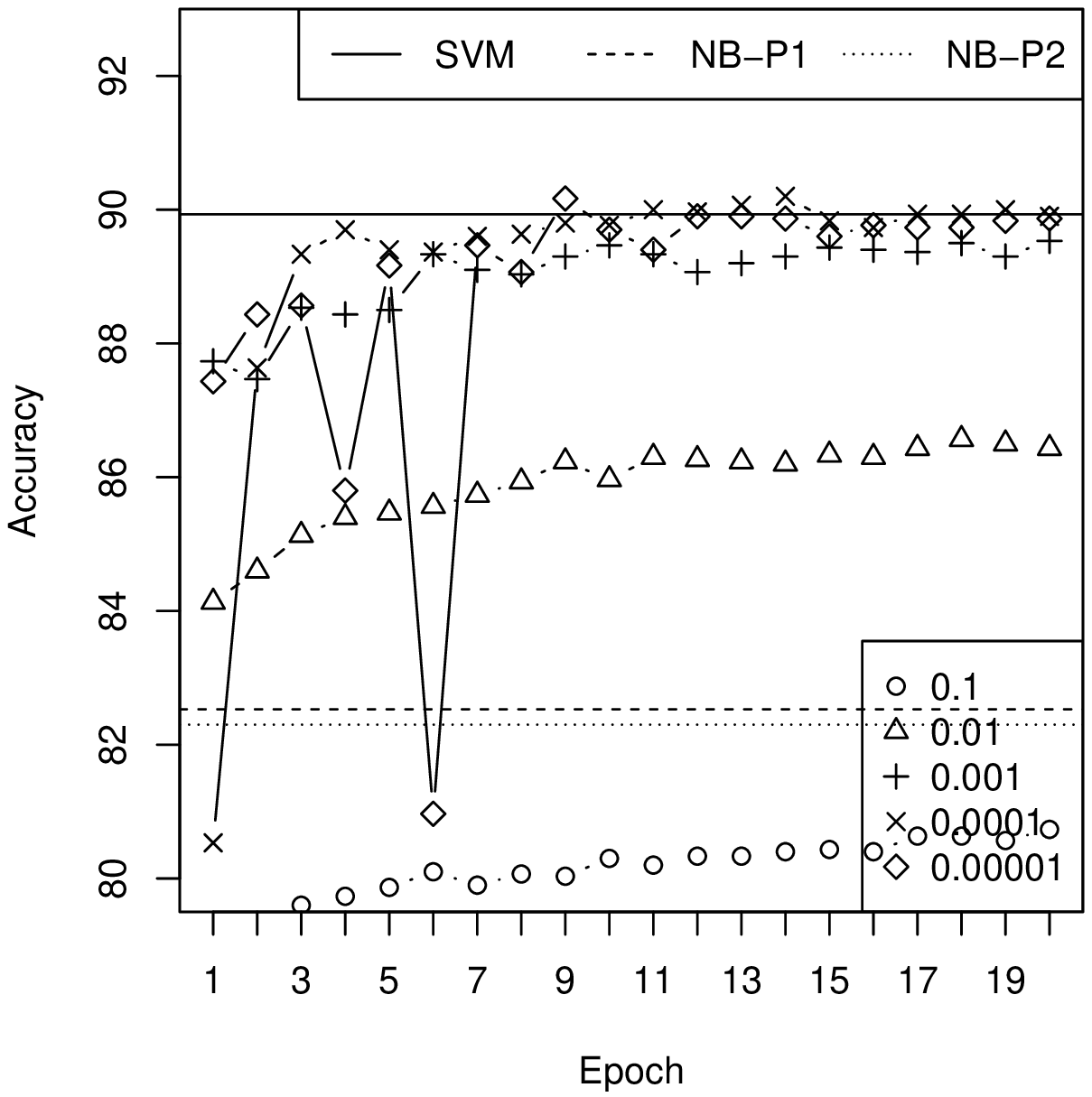,scale=0.43}\\
(a) Hinge-SGD & (b) Hinge-MNB with prior P1 & (c) Hinge-MNB with prior P2 \\
\end{tabular}
\end{center}
\end{figure*}

\section{Latent Dirichlet Allocation (LDA) for text classification}\label{Section:LDA}

\subsection*{Description of the algorithm}

As commented in the main paper,  we depart from a classification model similar to MNB, but where the documents of the same class are now modeled using an independent LDA model instead of a multinomial distribution. The generative process of each label-document pair in the corpus would be as follows \cite{blei2003latent}:

\begin{enumerate}
\item Choose class label $y\sim p(y|\theta_Y)$, a multinomial probability. 
\item Choose $N\sim Poisson(\xi_y)$, the length of the document follows a Poisson distribution.
\item Choose $\phi_y\sim Dir(\alpha_y)$, a Dirichlet distribution with dimension $|Z|$ (the meta-parameters are set to $1/|Z|$ in the experimental evaluation).
\item For each of the $N$ words $w_n$:
\begin{enumerate}
\item Choose a topic $z_n\sim Multinomial(\phi_y)$ with dimension $|Z|$.
\item Choose a word $w_n\sim p(w_n|z_n,\beta_y)$, a multinomial probability conditioned on the topic
$z_n$.
\end{enumerate}
\end{enumerate}

In our case the unknown parameters are the $\beta_y$ for each class label, which defines the multinomial distribution of the step 4 (b) and the parameter $\theta_Y$ which defines the prior. 

We denote by $d$ to a document as a bag of words $d=\{w_1,\ldots,w_N\}$ and we denote by $z_d$ to a particular hidden topic assignment vector for the words in $d$. Then the sufficient statistics for this model would be a three dimensional matrix indexed by $k\in\{1,...,|Y|\}$, $z\in\{1,\ldots,|Z|\}$ and $w \in \{1,\ldots,|W|\}$, where $|W|$ denotes again the total number of different words in the corpus. The $(k,z,w)$-th component of this sufficient statistics matrix is computed as follows:

$$s_{k,z,w}(y,z_d,d) = I[y=k]\sum_n I[z_n=z]I[w_n=w]$$

As previously commented in the main paper, these sufficient statistics would correspond to the "words per hidden topic counts". By adding the "prior class counts", we would complete all the sufficient statistics that define this classification model.  

As also commented in the main paper, similarly to \cite{RohdeCappe2011}, we used an online Collapsed Gibbs sampling method to obtain, at convergence, unbiased estimates of the expected sufficient statistics (see Section 3.5 in the main paper). This collapsed Gibbs sampling method makes used of the analytical marginalization of the parameter  $\phi_y$ and samples in turn each of the indicator variables $z_1,\ldots,z_N$.  The probability of an indicator variable $z_n$  conditioned on all the words of the document and all the other indicators variables can be computed as follows:

\begin{equation}\label{Equation:Gibbs}
p(z_n|y,\{z_{n'}\}_{n'\neq n},d) \propto \beta_{y,z_n,w_n} \cdot(S^{(-w_n)}_{z_n} + \alpha)
\end{equation}
\noindent where $S^{(-w_n)}_z = \sum_{n'\neq n} I[z_{n'}= z]$ and $\beta_{y,z_n,w_n}$ is the component of the $\beta$ parameter vector which defines the probability that the $n$-th word in document $d$ is equal to $w_n$ given that hidden topic is $z_n$ and the class label of the document is $y$, $p(w_n|z_n,\beta_y)$.  

The above equation defines a Markov chain that when it is run generates unbiased samples from its stationary distribution, $p(z_n|d,y)$ (after discarding the first burn-in samples). So, we could then compute the expected sufficient statistics required to apply the sdEM algorithm over these models.  Let us note that under our online settings the $\beta$ parameter of Equation (\ref{Equation:Gibbs}) is fixed to the values $\beta_{t-1}$ estimated in the previous step and the this online collapsed Gibbs sampler only requires that the simulation is conditioned to the latent variables of the current observed document (i.e. it does not involve the hidden topics of the other documents in the corpus as happens with its batch counterpart). 

In Algorithm \ref{alg:sdEM_LDA_NCLL} and Algorithm \ref{alg:sdEM_LDA_Hinge}, we give a pseudocode description of the sdEM algorithm when applied to the this LDA classification model when using the NCLL and the Hinge loss functions, respectively. As can be seen, this algorithms does not directly relate to the standard LDA implementation, because we employ the same simplification used in the implementation\footnote{Code available at http://www.cs.cmu.edu/~chongw/slda/}  of the sLDA algorithm \cite{blei2007supervised} for multi-class prediction. This simplification assumes that all the occurrences of the same word in a document share the same hidden topic. The first effect of this assumption is that the number of hidden variables is reduced and the algorithm is much quicker. Whether this simplifying assumption has a positive or negative effect in the classification performance of the models is not evaluated here.

Let us also see in the pseudo-code of these two algorithms, that Hinge-LDA will tend to be computationally more efficient than NCLL-LDA, because Hinge-LDA does not update any parameter when it classifies a document with a margin higher than 1. However, NCLL-LDA always updates all the parameters. When we deal with a high number of classes, this may imply a great difference  in the computational performance. But this is something which is not evaluated in this first experimental study.

We also use a heuristic method\footnote{It is proposed in http://shuyo.wordpress.com/2011/06/27/collapsed-gibbs-sampling-estimation-for-latent-dirichlet-allocation-3/.} to initialize the hidden topics variables $z_n$ of the incoming document which consists in sampling the hidden topics according to Equation \ref{Equation:Gibbs}, where $S^{(-w_n)}_{z_n}$ is computed on-the-fly i.e. for the first word is a vector of zeros and, then, it is updated according to the sampled topics. It is similar to running collapsed Gibbs sampling for one iteration.

We emphasis again that these algorithms are based on the updating equations given in the Table 2 of the main paper.

\begin{figure*}[!htbp]
\caption{\label{Figure:NCLL_LDA} Convergence behavior of sdEM when applied to the  LDA classification model. Left side figures consider the NCLL-LDA model, i.e. when sdEM minimizes the NCLL loss function. Then, the series NCLL-disc and Perplexity-disc display the evolution of these two losses for the NCLL-LDA model. Right side figures consider the Hinge-LDA model, i.e. when sdEM minimizes the Hinge loss function. Then, the series Hinge-disc and Perplexity-disc display the evolution of these two losses for the Hinge-LDA model. Series NCLL-gen, Hinge-gen and Perpelexity-gen show the evolution of the NCLL, Hinge and perplexity losses, respectively, for the NLL-LDA model, i.e. when sdEM minimizes the negative log-likelihood (NLL) loss function.}
\begin{center}
\begin{tabular}{cc}
\multicolumn{2}{c}{\textbf{Reuters-R8 data set}}\\\\
\epsfig{file=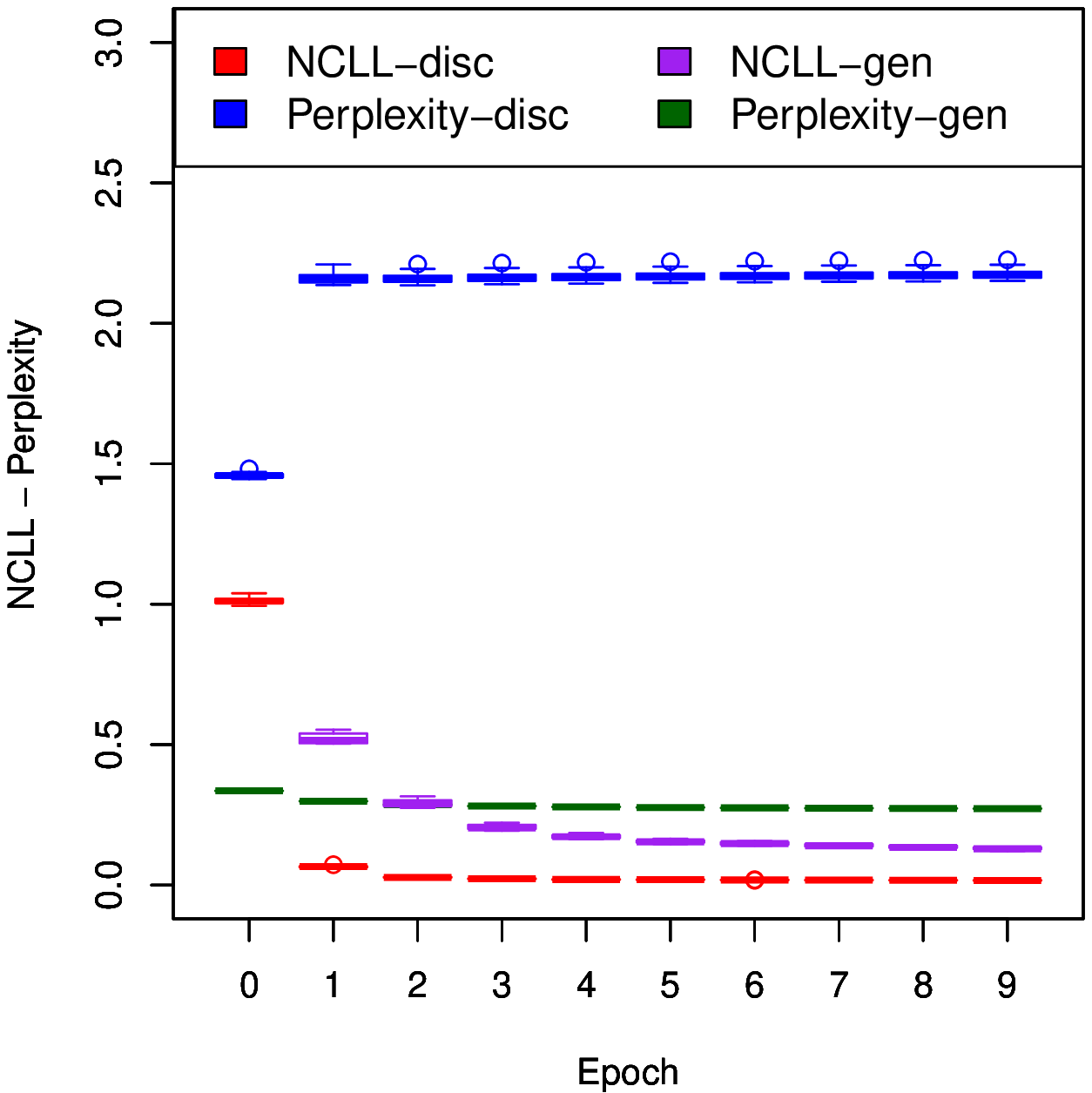,scale=0.6}&
\epsfig{file=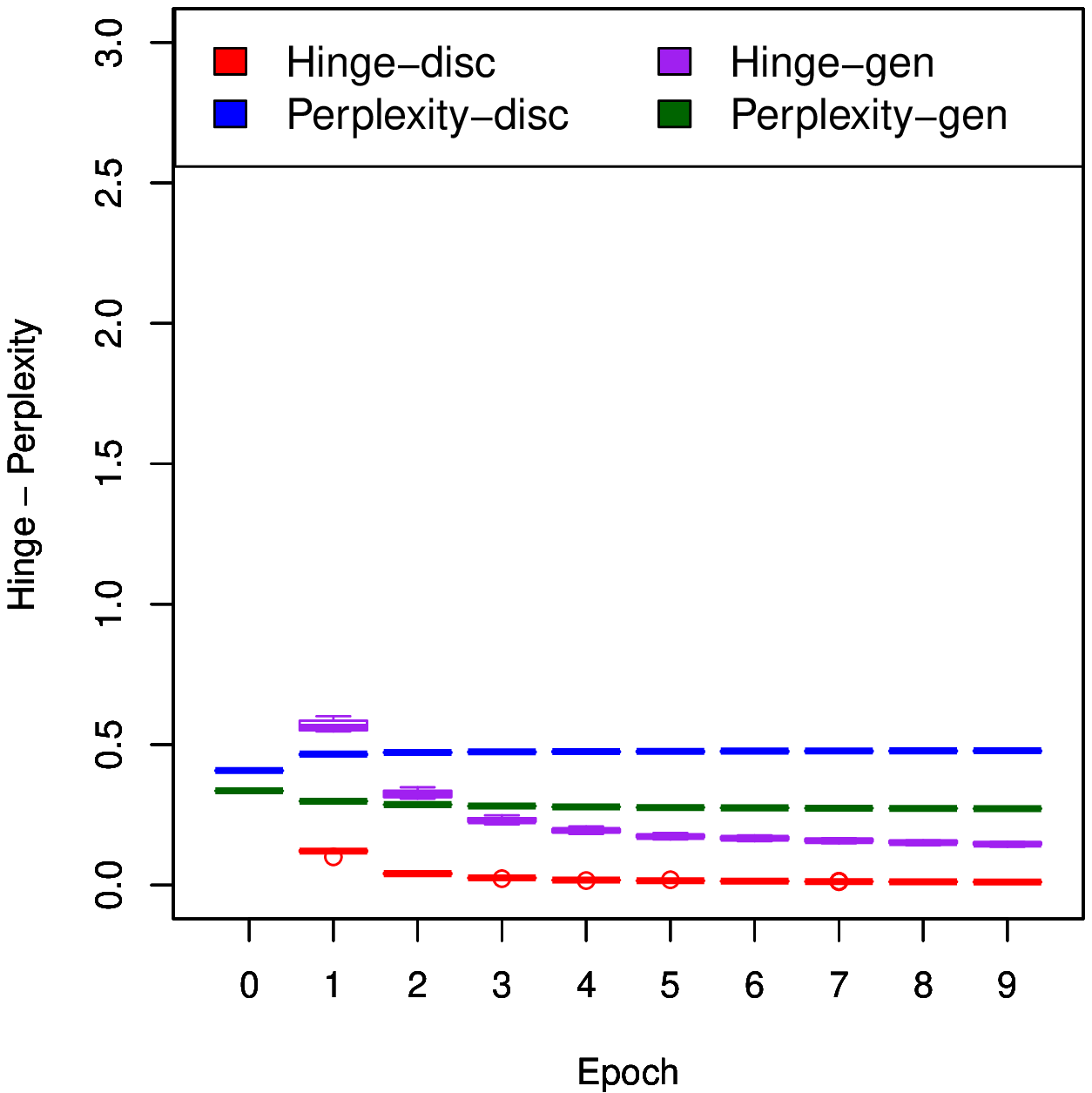,scale=0.6}\\\\
\multicolumn{2}{c}{\textbf{Web-KB data set}}\\\\
\epsfig{file=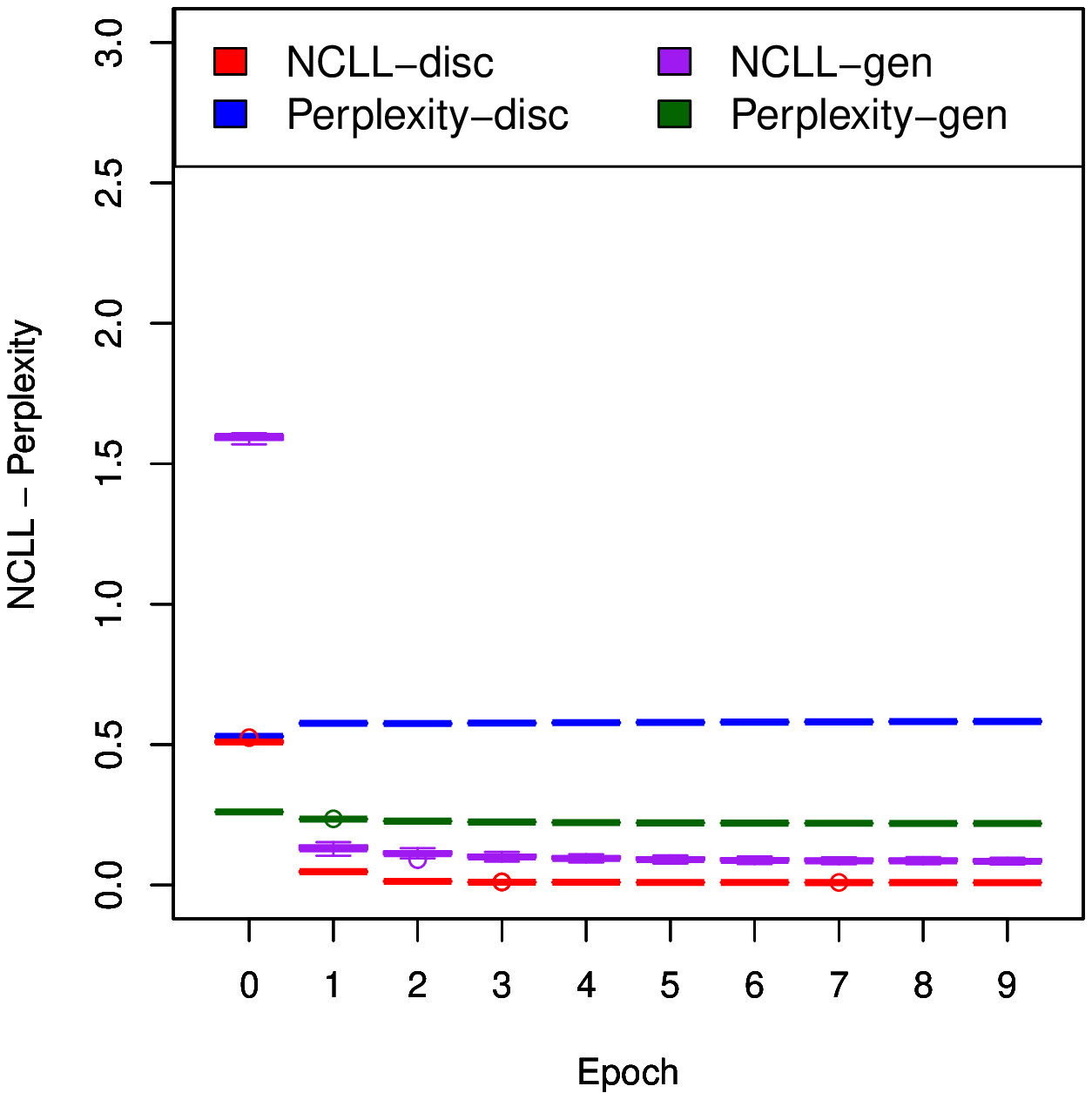,scale=0.6}&
\epsfig{file=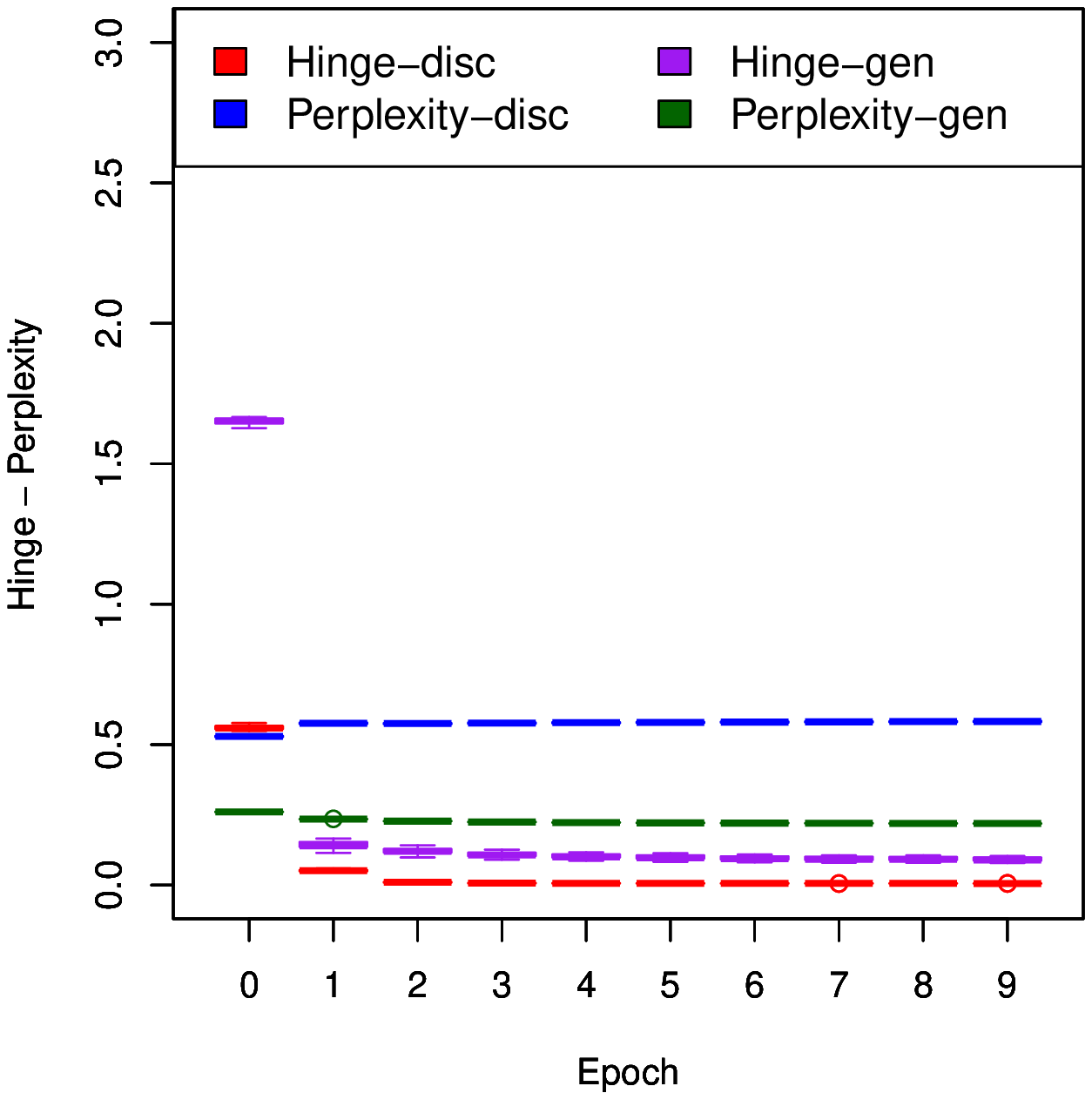,scale=0.6}\\
(a) NCLL-LDA  & (b) Hinge-LDA \\
\end{tabular}
\end{center}
\end{figure*}

\subsection*{Experimental Evaluation}

As previously commented in the paper, this evaluation was carried out using the standard train/test split of the Reuters21578-R8  and Web-KB  data sets
\cite{2007:phd-Ana-Cardoso-Cachopo}, under the same preprocessing than in the MNB's experiments. In Table \ref{Table:Data} some statistics about these data sets are given. 

We used sdEM to train 2-topics LDA classification by minimizing the NCLL loss (NCLL-LDA), by minimizing the
Hinge loss (Hinge-LDA), and also by minimizing the negative log-likelihood loss (NLL-LDA),
following the updating equations of Table 2 in the main paper.  We remind that at Figure 3 in the main paper, we show the results of the comparison of the classification accuracy of these models with the results obtained by supervised-LDA (sLDA) \cite{blei2007supervised} using the same prior, but using 50 topics because with less topics it produced worse results. 

We plot here at Figure \ref{Figure:NCLL_LDA}, the convergence behavior at the training phase of the above models. The aim is to highlight how there is again a similar trade-off between the different losses when we train this model by minimizing a discriminative loss function such as NCLL or Hinge loss w.r.t. when we train this same model by minimizing a "generative loss" such as the negative log-likelihood (NLL).

Looking at these figures we can see like neither NCLL-LDA nor Hinge-LDA decrease the perplexity loss in opposite to NLL-LDA. We can also see that NLL-LDA does decrease either the NCLL or the Hinge loss but not so successfully as NCLL-LDA or Hinge-LDA.

\section{sdEM Java Code}\label{Section:Code}
All the code used to build all the experiments presented in this supplemental material or in the main paper can be downloaded from the following code repository "https://sourceforge.net/projects/sdem/" (in "Files" tab). This code is written in Java and mostly builds on Weka \cite{hall2009weka} data structures.


\begin{algorithm*}[!htbp]
\caption{sdEM for Multinomial Naive Bayes with the NCLL loss. $|d|$ denotes the number of different words in the current document $d$ and $|w|_d$ to the number of times word $w$ appears in document $d$. $|W|$ denotes the total number of different words in the corpus.} \label{alg:sdEM_MNB_NCLL}
\begin{algorithmic}[1]
\REQUIRE  $D$ is randomly shuffled.
\REQUIRE  $\alpha$ value as prior count for each word. Two values are considered $\alpha=1$ and $\alpha=\ln |W|$.
\STATE $\forall k,w$ $N[k][w]=\alpha$; $C[k]=1.0$; $M[k]=\alpha*|W|$; 
\STATE $t=0$;
\STATE $\gamma = 0$;

\REPEAT

\FOR{ each label-document pair  $(y,d)$}
\STATE $t=t + 1$;
\STATE $\rho=\frac{1}{1+\lambda\cdot t}$
\STATE $\gamma = \gamma  + \alpha\cdot\frac{\rho}{n}$

\FOR{each distinct word $w$ in the document $d$} 
\STATE $N[y][w] = N[y][w] + \rho\cdot|w|_d\cdot(1-p(Y=y|d,N,M,C,\gamma))$;
\STATE $M[y] = M[y] + \rho\cdot|w|_d\cdot(1-p(Y=y|d,N,M,C,\gamma))$;
\FOR {$k=1,...,|Y|: k\neq y$}
\STATE $oldVal = N[k][w]$;
\STATE $N[k][w] =  N[k][w] - \rho\cdot|w|_d\cdot p s(Y=k|d,N,M,C,\gamma )$;
\STATE $N[k][w] = \max(N[k][w],0)$;
\STATE $M[k] = M[k] + (N[k][w] - oldVal)$;
\ENDFOR
\ENDFOR

\STATE $C[y] = C[y]  + \rho\cdot(1-p(Y=y|d,N,M,C,\gamma))$; 
\FOR {$k=1,...,|Y|: k\neq y$}
\STATE $C[k] = C[k]  - \rho\cdot p(Y=k|d,N,M,C,\gamma )$;
\STATE $C[k] = \max(C[k],0)$;
\ENDFOR
\ENDFOR

\UNTIL convergence

\STATE $\bar{N} = Normalize(N,\gamma);$
\STATE $\bar{C} = Normalize(C,\gamma);$

\RETURN $\bar{N}$ and $\bar{C}$;

\end{algorithmic}
\end{algorithm*}
 
\begin{algorithm*}[!htbp]
\caption{Compute predictions $P(Y=k|d,N,M,C,\gamma)$ with Multinomial Naive Bayes. $|d|$ denotes the number of different words in the current document $d$ and $|w|_d$ denotes the number of times word $w$ appears in document $d$. The function "Logs2Probs" simply exponentiate the log values and then normalize.} \label{alg:sdEM_MNB_Posterior}
\begin{algorithmic}[1]
\REQUIRE  $N$, $M$, $C$, $\gamma$ with non-negative values.

\STATE $\forall k$ $LogDC[k]=0.0$;
\FOR {$k=1,...,|Y|$}
\STATE $LogDC[k] = \ln(C[k] + \gamma)$;
\STATE $sumW=0$;
\FOR{each distinct word $w$ in the document $d$} 
\STATE $LogDC[k] = LogDC[k] + |w|_d\cdot\ln(N[y][w] +\gamma)$;
\STATE $sumW = sumW + |w|_d$;
\ENDFOR
\STATE $LogDC[k] = LogDC[k]  - sumW\cdot\ln(M[k] +\gamma\cdot|d|)$;
\ENDFOR

\RETURN Logs2Probs(LogDC);

\end{algorithmic}
\end{algorithm*}

\begin{algorithm*}[!htbp]
\caption{sdEM for Multinomial Naive Bayes with the Hinge loss. $|d|$ denotes the number of different words in the current document $d$ and $|w|_d$ denotes the number of times word $w$ appears in document $d$. $|W|$ denotes the total number of different words in the corpus.} \label{alg:sdEM_MNB_Hinge}
\begin{algorithmic}[1]
\REQUIRE  $D$ is randomly shuffled.
\REQUIRE  $\alpha$ value as prior count for each word. Two values are considered $\alpha=1$ and $\alpha=\ln |W|$.
\STATE $\forall k,w$ $N[k][w]=\alpha$; $C[k]=1.0$; $M[k]=\alpha*|d|$; 
\STATE $t=0$;
\STATE $\gamma = 0$;

\REPEAT

\FOR{ each label-document pair $(y,d)$}
\STATE $t=t + 1$;
\STATE $\rho=\frac{1}{1+\lambda\cdot t}$
\STATE $\gamma = \gamma  +\alpha\cdot\frac{\rho}{n}$
\STATE $\bar{y} = arg\max_{y'\neq y} p(Y=y'|x)$;
\IF{$(\ln p(Y=y|x) - \ln p(Y=\bar{y}|x)) > 1$}
\STATE Go for the next document;
\ENDIF

\FOR{each distinct word $w$ in the document $d$} 
\STATE $N[y][w] = N[y][w] + \rho\cdot|w|_d\cdot(1-p(Y=y|d,N,M,C,\gamma))$;
\STATE $M[y] = M[y] + \rho\cdot|w|_d\cdot(1-p(Y=y|d,N,M,C,\gamma))$;
\STATE $oldVal = N[\bar{y}][w]$;
\STATE $N[\bar{y}][w] = N[\bar{y}][w] - \rho\cdot|w|_d\cdot p(Y=\bar{y}|d,N,M,C,\gamma)$;
\STATE $N[\bar{y}][w] = \max(N[\bar{y}][w],0)$;
\STATE $M[k] = M[k] + (N[\bar{y}][w] - oldVal)$;
\ENDFOR

\STATE $C[y] = C[y]  + \rho\cdot(1-p(Y=y|d,N,M,C,\gamma))$; 
\STATE $C[\bar{y}] = C[\bar{y}]  - \rho\cdot p(Y=\bar{y}|d,N,M,C,\gamma)$; 
\STATE $C[\bar{y}] = \max(C[\bar{y}],0)$;
\ENDFOR

\UNTIL convergence

\STATE $\bar{N} = Normalize(N,\gamma);$
\STATE $\bar{C} = Normalize(C,\gamma);$

\RETURN $\bar{N}$ and $\bar{C}$;

\end{algorithmic}
\end{algorithm*}

\begin{algorithm*}[!htbp]
\caption{sdEM for the LDA based classifier using the NCLL loss. $|d|$ denotes the number of different words in the current document $d$, $|w|_d$ denotes the number of times word $w$ appears in document $d$ and $|Z|$ denotes the number of hidden topics in the LDA model.} \label{alg:sdEM_LDA_NCLL}
\begin{algorithmic}[1]
\REQUIRE  $D$ is randomly shuffled.
\REQUIRE  $\eta$ defines the prior for the "word per topic counts". In the experiments, it is fixed to $\eta=0.1$. 
\STATE $\forall k,z,w$ $N[k][z][w]=\frac{\eta}{|Z|}$; $C[k]=1.0$; $M[k][z]=|W|\frac{\eta}{|Z|}$; 
\STATE $t=0$;
\STATE $\gamma = 0$;

\REPEAT

\FOR{ each label-document pair  $(y,d)$}
\STATE $t=t + 1$;
\STATE $\rho=\frac{1}{1+\lambda\cdot t}$
\STATE $\gamma = \gamma  + \frac{\eta}{|Z|}\cdot\frac{\rho}{n}$
\STATE OnlineLDA($d$,$N[y]$,$M[y]$,$\rho$, $\gamma$, $\varpi =  (1-p(Y=y|d,N,M,C,\gamma))$);
\FOR {$k=1,...,|Y|: k\neq y$}
\STATE OnlineLDA($d$,$N[k]$,$M[k]$,$\rho$, $\gamma$, $\varpi =-p(Y=k|d,N,M,C,\gamma)$);
\ENDFOR
\STATE $C[y] = C[y]  + \rho\cdot(1-p(Y=y|d,N,M,C,\gamma))$; 
\FOR {$k=1,...,|Y|: k\neq y$}
\STATE $C[k] = C[k]  - \rho\cdot p(Y=k|d,N,M,C,\gamma )$;
\STATE $C[k] = \max(C[k],0)$;
\ENDFOR
\ENDFOR

\UNTIL convergence

\STATE $\bar{N} = Normalize(N,\gamma);$
\STATE $\bar{C} = Normalize(C,\gamma);$

\RETURN $\bar{N}$ and $\bar{C}$;

\end{algorithmic}
\end{algorithm*}

\begin{algorithm*}[!htbp]
\caption{OnlineLDA($d$,$N$,$M$,$\rho$, $\gamma$, $\varpi$). The vector $s$ would correspond to the expected sufficient statistics for $d$ computed by  online collpased Gibbs sampling.}\label{alg:OnlineLDA}
\begin{algorithmic}[1]
\REQUIRE $d$,$N$,$M$,$\rho$, $\gamma$, $\varpi$ properly computed.

\STATE s = OnlineCollapsedGibbsSampling(d,N,M,$\gamma$);
\FOR{each distinct word $w$ in the document $d$} 
\FOR{z=1,...,$|Z|$}
\STATE $oldVal =N[z][w]$;å 
\STATE $N[z][w] =  N[z][w] + \rho\cdot \varpi \cdot s[z][w]$;
\STATE $N[z][w] = \max(N[z][w],0)$;
\STATE $M[z] = M[z] + (N[z][w]-oldVal)$;
\ENDFOR
\ENDFOR

\end{algorithmic}
\end{algorithm*}

\begin{algorithm*}[!htbp]
\caption{sdEM for the LDA based classifier using the Hinge loss. $|d|$ denotes the number of different words in the current document $d$, $|w|_d$ denotes the number of times word $w$ appears in document $d$ and $|Z|$ denotes the number of hidden topics in the LDA model.} \label{alg:sdEM_LDA_Hinge}
\begin{algorithmic}[1]
\REQUIRE  $D$ is randomly shuffled.
\REQUIRE  $\eta$ defines the prior for the "word per topic counts". In the experiments, it is fixed to $\eta=0.1$. 
\STATE $\forall k,z,w$ $N[k][z][w]=\frac{\eta}{|Z|}$; $C[k]=1.0$; $M[k][z]=|W|\frac{\eta}{|Z|}$; 
\STATE $t=0$;
\STATE $\gamma = 0$;

\REPEAT

\FOR{ each label-document pair $(y,d)$}
\STATE $t=t + 1$;
\STATE $\rho=\frac{1}{1+\lambda\cdot t}$
\STATE $\gamma = \gamma  + \frac{\eta}{|Z|}\cdot\frac{\rho}{n}$
\STATE $\bar{y} = arg\max_{y'\neq y} p(Y=y'|x)$;
\IF{$(\ln p(Y=y|x) - \ln p(Y=\bar{y}|x)) > 1$}
\STATE Go for the next document;
\ENDIF
\STATE OnlineLDA($d$,$N[y]$,$M[y]$,$\rho$, $\gamma$, $\varpi =  (1-p(Y=y|d,N,M,C,\gamma))$);
\STATE OnlineLDA($d$,$N[\bar{y}]$,$M[\bar{y}]$,$\rho$, $\gamma$, $\varpi =-p(Y=\bar{y}|d,N,M,C,\gamma)$);
\STATE $C[y] = C[y]  + \rho\cdot(1-p(Y=y|d,N,M,C,\gamma))$; 
\STATE $C[\bar{y}] = C[\bar{y}]  - \rho\cdot p(Y=\bar{y}|d,N,M,C,\gamma)$; 
\STATE $C[\bar{y}] = \max(C[\bar{y}],0)$;
\ENDFOR

\UNTIL convergence

\STATE $\bar{N} = Normalize(N,\gamma);$
\STATE $\bar{C} = Normalize(C,\gamma);$

\RETURN $\bar{N}$ and $\bar{C}$;

\end{algorithmic}
\end{algorithm*}

\end{appendices}

\end{document}